\documentclass[12pt]{article}
\usepackage{geometry,mathtools}
\usepackage{amsmath,amssymb,amsthm, amsfonts}
\usepackage{etaremune, enumerate, float, verbatim}
\usepackage{algorithm}
\usepackage{algorithmic}
\usepackage{hyperref}
\usepackage{graphicx}
\usepackage{url}
\usepackage[mathlines]{lineno}
\usepackage{dsfont} 
\usepackage{tikz, graphicx, subcaption, caption}
\usepackage[algo2e,ruled,vlined]{algorithm2e}
\usepackage{mathrsfs,amsmath}
\usepackage{color}
\definecolor{red}{rgb}{1,0,0}

\definecolor{blue}{rgb}{0,0,.7}

\definecolor{green}{rgb}{0,.6,0}

\definecolor{purp}{rgb}{.5,0,.5}

\numberwithin{figure}{section}   

\DeclareMathOperator*{\esssup}{ess\,sup}

\newtheorem{thm}{Theorem}[section]
\newtheorem{cor}[thm]{Corollary}
\newtheorem{lem}[thm]{Lemma}
\newtheorem{prop}[thm]{Proposition}

\newtheorem{obs}[thm]{Observation}
 
\newtheorem{prob}[thm]{Problem}

\theoremstyle{definition}

\theoremstyle{definition}
\newtheorem{defn}[thm]{Definition}

\theoremstyle{definition}

\newcommand{\opt}{\operatorname{opt}}

\newcommand{\linint}{\operatorname{LININT}}

\newcommand{\id}{\operatorname{id}}

\newcommand{\bit}{\begin{itemize}}
\newcommand{\eit}{\end{itemize}}
\newcommand{\ben}{\begin{enumerate}}
\newcommand{\een}{\end{enumerate}}
\newcommand{\beq}{\begin{equation}}
\newcommand{\eeq}{\end{equation}}
\newcommand{\bea}{\begin{eqnarray*}} 
\newcommand{\eea}{\end{eqnarray*}}
\newcommand{\bpf}{\begin{proof}}
\newcommand{\epf}{\end{proof}\ms}
\newcommand{\bmt}{\begin{bmatrix}}
\newcommand{\emt}{\end{bmatrix}}
\newcommand{\ms}{\medskip}

\newcommand{\noi}{\noindent}

\newcommand{\R}{\mathbb{R}}

\title{Online learning of smooth functions on $\mathbb{R}$}
\author{Jesse Geneson, Kuldeep Singh, and Alexander Wang}

\begin{document}
\maketitle

\begin{abstract}
We study adversarial online learning of real-valued functions on $\mathbb{R}$. In each round the learner is
queried at $x_t\in\mathbb{R}$, predicts $\hat y_t$, and then observes the true
value $f(x_t)$; performance is measured by cumulative $p$-loss
$\sum_{t\ge 1}|\hat y_t-f(x_t)|^p$. For the class
\[
\mathcal{G}_q=\Bigl\{f:\mathbb{R}\to\mathbb{R}\ \text{absolutely continuous}:\ 
\int_{\mathbb{R}}|f'(x)|^q\,dx\le 1\Bigr\},
\]
we show that the standard model becomes ill-posed on $\mathbb{R}$: for every
$p\ge 1$ and $q>1$, an adversary can force infinite loss. Motivated by this obstruction, we analyze three modified learning scenarios
that limit the influence of queries that are far from previously observed inputs.
In Scenario~1 the adversary must choose each new query within distance $1$ of
some past query. In Scenario~2 the adversary may query anywhere, but the learner
is penalized only on rounds whose query lies within distance $1$ of a past
query. In Scenario~3 the loss in round $t$ is multiplied by a weight
$g(\min_{j<t}|x_t-x_j|)$.

We obtain sharp characterizations for Scenarios~1--2 in several regimes. For Scenario~3 we identify a clean threshold phenomenon: if
$g$ decays too slowly, then the adversary can force infinite
weighted loss. In contrast, for rapidly decaying weights such as
$g(z)=e^{-cz}$ we obtain finite and sharp guarantees in the quadratic case
$p=q=2$. Finally, we study a natural multivariable slice generalization $\mathcal{G}_{q,d}$
of $\mathcal{G}_q$ on $\mathbb{R}^d$ and show a sharp dichotomy: while the
one-dimensional case admits finite opt-values in certain regimes, for every
$d\ge 2$ the slice class $\mathcal{G}_{q,d}$ is too permissive, and even under
Scenarios~1--3 an adversary can force infinite loss.
\end{abstract}

\noi {\bf Keywords}: online learning; smooth functions; unbounded domains; worst-case error bounds

\section{Introduction}

Consider the model of online learning of smooth functions from \cite{GZ, klsf, smooth2, xie}, which is a variant of the mistake-bound model of online learning (see, e.g., \cite{angluin,almw,auer,feng,filmus,gen,gt,littlestone,long}). An algorithm $A$ tries to learn a real-valued function $f$ from some class $\mathcal{F}$ with domain $[0,1]$. In each trial of the model, $A$ receives an input $x_t \in [0, 1]$, it must output some prediction $\hat{y}_t$ for $f(x_t)$, and then $A$ discovers the true value of $f(x_t)$. 

For each $p \ge 1$ and $x = (x_0, \dots, x_m) \in [0, 1]^{m+1}$, define $L_{p,[0,1]}(A, f, x) = \sum_{t = 1}^m |\hat{y}_t-f(x_t)|^p$. Note that the summation starts on the second trial, since the guess on the first trial does not reflect the algorithm's learning ability. Define \[L_{p,[0,1]}(A, \mathcal{F}) = \displaystyle \sup_{f \in \mathcal{F}, x \in  \cup_{m \in \mathbb{N}} [0,1]^{m+1}} L_{p,[0,1]}(A, f, x).\] Define the optimum $\opt_{p,[0,1]}(\mathcal{F}) = \displaystyle \inf_A L_{p,[0,1]}(A, \mathcal{F})$. 

Past research on this topic has focused mostly on the class of functions whose first derivatives have $q$-norms at most $1$. For any real number $q \ge 1$, let $\mathcal{F}_q$ be the family of absolutely continuous functions $f: [0, 1] \rightarrow \mathbb{R}$ for which $\int_{0}^1 |f'(x)|^q dx \le 1$. Let $\mathcal{F}_{\infty}$ be the family of absolutely continuous functions $f:  [0, 1] \rightarrow \mathbb{R}$ for which $\displaystyle \sup_{x \in [0, 1]} |f'(x)| \le 1$. By Jensen's inequality, we have $\mathcal{F}_{\infty} \subseteq \mathcal{F}_r \subseteq \mathcal{F}_q$ for any $1 \le q \le r$. Thus, \[\opt_{p,[0,1]}(\mathcal{F}_{\infty}) \le \opt_{p,[0,1]}(\mathcal{F}_r) \le \opt_{p,[0,1]}(\mathcal{F}_q)\] for any $1 \le q \le r$. 

Kimber and Long \cite{klsf} showed that $\opt_{p,[0,1]}(\mathcal{F}_q) = 1$ for all $p, q \ge 2$. They also showed that $\opt_{1,[0,1]}(\mathcal{F}_q) = \infty$ for all $q \ge 1$, $\opt_{1,[0,1]}(\mathcal{F}_{\infty}) = \infty$, and $\opt_{p,[0,1]}(\mathcal{F}_1) = \infty$ for all $p \ge 1$. Geneson and Zhou \cite{GZ} showed that $\opt_{p,[0,1]}(\mathcal{F}_q) = 1$ for all $p, q$ with $q > 1$ and $p \ge 2+\frac{1}{q-1}$, $\opt_{1+\varepsilon,[0,1]}(\mathcal{F}_q) = \Theta(\varepsilon^{-1/2})$ for all $q \ge 2$, and $\opt_{1+\varepsilon,[0,1]}(\mathcal{F}_{\infty}) = \Theta(\varepsilon^{-1/2})$. They also introduced a generalization of the problem to multivariable functions. Most of the results in these papers can be generalized to learning scenarios where the domain $[0,1]$ is replaced by any real interval $[a,b]$. In this paper, we consider what happens when the domain is unbounded.

\medskip
\noindent\textbf{Online learning on unbounded domains.}
The restriction to a compact domain such as $[0,1]$ is mathematically convenient,
but it is often an artifact of modeling rather than a reflection of how online
prediction problems arise in practice.
In many natural settings, inputs are not confined to a fixed interval and may
grow, drift, or be selected adaptively over time.
Examples include time-indexed or scale-indexed regression problems, adaptive
scientific measurement and sensing, online control and tuning of continuous
parameters, and sequential decision systems facing persistent distribution
shift.
In such settings, smoothness assumptions are often global in nature, while the domain
itself is effectively unbounded.

From a modeling perspective, unbounded domains expose a fundamental tension
between smoothness and extrapolation.
While smoothness controls local variation, it does not by itself prevent an
adversary from placing queries arbitrarily far from previously observed inputs.
As we show in Section~\ref{sec:2}, this makes the most direct extension of the
classical mistake-bound model ill-posed on $\mathbb{R}$: even extremely smooth
functions admit adversarial strategies that force infinite loss.
This phenomenon highlights a structural limitation of worst-case online learning
when extrapolation is unconstrained.

At the same time, real systems rarely treat all extrapolations equally.
Predictions far from previously observed data are often regarded as exploratory,
less reliable, or subject to reduced accountability.
This observation motivates the alternative learning scenarios studied in this
paper.
Scenario~1 enforces locality by constraining how far the adversary may move
between successive queries.
Scenario~2 preserves unrestricted inputs but evaluates the learner only on
queries that are sufficiently close to past observations.
Scenario~3 interpolates between these extremes by weighting errors according to
their distance from previously seen inputs, formalizing the intuition that
confidence should decay with extrapolation distance.

Together, these scenarios provide a principled framework for understanding which
forms of locality are necessary and which are sufficient to recover meaningful
worst-case guarantees for smooth function learning on unbounded domains.

\medskip
\noindent{\bf Results.}
In Section~\ref{sec:2} we formulate the mistake-bound model on the unbounded
domain $\mathbb{R}$ and introduce the classes $\mathcal{G}_q$.
We first show that the familiar nesting relations from bounded intervals break
down on $\mathbb{R}$ (Proposition~\ref{prop:all-noninclusions}), and then prove
that the naive extension of the standard model is ill-posed: for every
$p\ge 1$ and $q>1$ an adversary can force infinite loss, so
$\opt_{p,\mathbb{R}}(\mathcal{G}_q)=\infty$ (Observation~\ref{obs:adv_inf}).
We also include a tractable unbounded-domain baseline beyond $\mathcal{G}_q$,
showing that truncated linear classes admit an exact optimum
(Theorem~\ref{thm:glnr}).

In Section~\ref{sec:3} we introduce the three locality-based mitigations
(Scenarios~1--3) and establish general comparisons between them, including
Scenario~1 versus Scenario~2 (Lemma~\ref{lem:sc12}), Scenario~2 versus
identity-weighted Scenario~3 (Lemma~\ref{lem:comp23}), and a general
weight-comparison principle for Scenario~3 (Lemma~\ref{lem:weight-compare}),
with useful consequences relating exponential and identity weights
(Corollary~\ref{cor:exp_id}) and comparing weighted to unweighted loss
(Corollary~\ref{cor:g-vs-unweighted}).
Section~\ref{sec:equal12} develops regimes where Scenarios~1 and~2 coincide,
including sharp guarantees for $\mathcal{G}_q$: when $p\ge q\ge 2$ we obtain
$\opt'_{p,\mathbb{R}}(\mathcal{G}_q)=\opt^*_{p,\mathbb{R}}(\mathcal{G}_q)=1$
via the modified interpolation algorithm $\linint'$ (Theorem~\ref{thm:1pgeq}
and its corollary), while for $0<p<q$ both scenarios still admit infinite loss
(Theorem~\ref{thm:inf_p<q}).  Section~\ref{sec:noteq} then shows Scenarios~1 and~2
can differ substantially for finite families: we give explicit constructions
with $\opt'_{p,\mathbb{R}}(F)<\opt^*_{p,\mathbb{R}}(F)$ (e.g.\ the families
$H_\varepsilon$, $F_\varepsilon$, and $F_{n,\varepsilon}$), obtain logarithmic
separations, and prove the general upper bound
$\opt^*_{p,\mathbb{R}}(F)/\opt'_{p,\mathbb{R}}(F)\le |F|-1$. In Section~\ref{sec:rd12}, we investigate Scenarios 1 and 2 with other choices of radius. We show for every $R>0$ that the corresponding
opt-values for $\mathcal{G}_q$ scale exactly as
$R^{(q-1)p/q}$ times their radius-$1$ counterparts.

Finally, Section~\ref{sec:sc3} focuses on Scenario~3: we prove sharp results for
$\mathcal{G}_2$ under identity weighting (Theorem~\ref{thm:s322}), identify an
exact constant for quadratic weighted loss in terms of $\sup_{z>0} z\,g(z)$,
and give a general divergence criterion for slowly decaying weights
(Proposition~\ref{prop:slowdecay}); we also analyze Scenario~3 for the truncated
linear classes, contrasting identity and exponential weights. We also investigate a multivariable extension of the unbounded-domain problem.
In Section~\ref{sec:Gqd-opt} we introduce the slice-based class $\mathcal{G}_{q,d}$,
a direct analogue of the multivariable classes studied in \cite{GZ}, and analyze
its behavior under Scenarios~1--3.
In sharp contrast to the one-dimensional case, we show that for every $d\ge 2$
the class $\mathcal{G}_{q,d}$ is fundamentally non-learnable on $\mathbb{R}^d$:
even with strong locality restrictions, the adversary can force infinite loss
for all $p>0$ and $q>1$. 

We conclude in Section~\ref{sec:conc} with a summary and future directions.

\section{Online learning with unbounded domain}\label{sec:2}

In this paper, we focus on learning scenarios where $[0,1]$ is replaced by $\mathbb{R}$. Specifically, in our most basic learning scenario, an algorithm $A$ tries to learn a real-valued function $f$ from some class $\mathcal{F}$ with domain $\mathbb{R}$. In each trial of this model, $A$ receives an input $x_t \in \mathbb{R}$, it must output some prediction $\hat{y}_t$ for $f(x_t)$, and then $A$ discovers the true value of $f(x_t)$. 

For each $p \ge 1$ and $x = (x_0, \dots, x_m) \in \mathbb{R}^{m+1}$, define $L_{p,\mathbb{R}}(A, f, x) = \sum_{t = 1}^m |\hat{y}_t-f(x_t)|^p$. Again note that the summation starts on the second trial. Define \[L_{p,\mathbb{R}}(A, \mathcal{F}) = \displaystyle \sup_{f \in \mathcal{F}, x \in  \cup_{m \in \mathbb{N}} \mathbb{R}^{m+1}} L_{p,\mathbb{R}}(A, f, x).\] We define the optimum $\opt_{p,\mathbb{R}}(\mathcal{F}) = \displaystyle \inf_A L_{p,\mathbb{R}}(A, \mathcal{F})$. 

As in \cite{klsf, smooth2}, we focus on smooth functions. For any real number $q \ge 1$, let $\mathcal{G}_q$ be the family of absolutely
continuous functions $f:\mathbb{R}\to\mathbb{R}$ such that
\[
\int_{-\infty}^{\infty} |f'(x)|^q\,dx \le 1.
\]
Let $\mathcal{G}_{\infty}$ be the family of absolutely continuous functions
$f:\mathbb{R}\to\mathbb{R}$ such that
\[
\operatorname*{ess\,sup}_{x\in\mathbb{R}} |f'(x)| \le 1.\]

\begin{prop}\label{prop:all-noninclusions}
Let $1\le q<r<\infty$. On $\mathbb{R}$, none of the derivative-norm classes
$\mathcal{G}_q$ are nested. More precisely, all of the following hold:
\begin{enumerate}
\item $\mathcal{G}_\infty \not\subseteq \mathcal{G}_q$ for every finite $q\ge 1$.
\item $\mathcal{G}_q \not\subseteq \mathcal{G}_\infty$ for every finite $q\ge 1$.
\item $\mathcal{G}_r \not\subseteq \mathcal{G}_q$.
\item $\mathcal{G}_q \not\subseteq \mathcal{G}_r$.
\end{enumerate}
\end{prop}

\begin{proof}
(1) Let $f(x)=x$. Then $f$ is absolutely continuous and $f'(x)=1$ for
every $x\in\mathbb{R}$, so $\operatorname*{ess\,sup}_{x\in\mathbb{R}}|f'(x)|=1$,
hence $f\in\mathcal{G}_\infty$. However, for any finite $q\ge 1$,
\[
\int_{-\infty}^{\infty} |f'(x)|^q\,dx
=
\int_{-\infty}^{\infty} 1\,dx
=\infty,
\]
so $f\notin\mathcal{G}_q$.

\smallskip
\noindent
(2) Fix finite $q\ge 1$. For each integer $n\ge 1$, let
\[
I_n := \Bigl[n,\; n+2^{-n(q+1)}\Bigr],
\qquad
\phi(x) := \sum_{n=1}^\infty 2^n\,\mathbf{1}_{I_n}(x),
\qquad
f(x):=\int_0^x \phi(t)\,dt.
\]
The intervals $I_n$ are disjoint, so $\phi$ is well-defined and locally
integrable. Hence $f$ is absolutely continuous and $f'(x)=\phi(x)$ for almost
every $x$. Moreover,
\[
\int_{-\infty}^{\infty} |f'(x)|^q\,dx
=
\sum_{n=1}^\infty \int_{I_n} (2^n)^q\,dx
=
\sum_{n=1}^\infty 2^{nq}\,|I_n|
=
\sum_{n=1}^\infty 2^{nq}\cdot 2^{-n(q+1)}
=
\sum_{n=1}^\infty 2^{-n}
=
1,
\]
so $f\in\mathcal{G}_q$. On the other hand, for every $M>0$ there exists $n$
with $2^n>M$, and then $I_n\subseteq\{x:|f'(x)|>M\}$ has positive measure.
Thus $\operatorname*{ess\,sup}_{x\in\mathbb{R}}|f'(x)|=\infty$, so
$f\notin\mathcal{G}_\infty$.

\smallskip
\noindent
(3) Define
\[
F'(x):=(1+|x|)^{-1/q},\qquad F(x):=\int_0^x F'(t)\,dt.
\]
Then $F$ is absolutely continuous. We have
\[
\int_{-\infty}^{\infty} |F'(x)|^q\,dx
=
2\int_0^\infty (1+x)^{-1}\,dx
=\infty,
\]
so $F\notin\mathcal{G}_q$. On the other hand,
\[
\int_{-\infty}^{\infty} |F'(x)|^r\,dx
=
2\int_0^\infty (1+x)^{-r/q}\,dx <\infty,
\]
since $r/q>1$. After scaling $F$ by a constant factor so that
$\int |F'|^r\le 1$, we obtain an element of $\mathcal{G}_r\setminus\mathcal{G}_q$.
Hence $\mathcal{G}_r\not\subseteq\mathcal{G}_q$.

\smallskip
\noindent
(4) Define
\[
G'(x):=
\begin{cases}
|x|^{-1/r}, & 0<|x|<1,\\
0, & |x|\ge 1,\\
0, & x=0,
\end{cases}
\qquad
G(x):=\int_0^x G'(t)\,dt.
\]
Then $G$ is absolutely continuous. We have
\[
\int_{-\infty}^{\infty} |G'(x)|^r\,dx
=
2\int_0^1 x^{-1}\,dx
=\infty,
\]
so $G\notin\mathcal{G}_r$. On the other hand,
\[
\int_{-\infty}^{\infty} |G'(x)|^q\,dx
=
2\int_0^1 x^{-q/r}\,dx <\infty,
\]
since $q/r<1$. After scaling $G$ by a constant factor so that
$\int |G'|^q\le 1$, we obtain an element of $\mathcal{G}_q\setminus\mathcal{G}_r$.
Hence $\mathcal{G}_q\not\subseteq\mathcal{G}_r$.
\end{proof}

For all $p, q \ge 1$, we prove in the following result that the adversary can force infinite error for $\opt_{p,\mathbb{R}}(\mathcal{G}_q)$. This leads us to consider more restrictive versions of the problem for functions with domain $\mathbb{R}$ in the next subsection. 

\begin{obs}\label{obs:adv_inf}
$\opt_{p,\mathbb{R}}(\mathcal{G}_q) = \infty$ for all $p \ge 1$ and $q > 1$.
\end{obs}

\begin{proof}
Let $c \ge 2$ and $q = 1+r$. The adversary can use the inputs $x_0 = 0$ and $x_i =x_{i-1}+ c^{i}$ for each $i \ge 1$, let $g(x_0) = 0$, and let $g(x_i) = g(x_{i-1})$ or $g(x_i) = g(x_{i-1})+h$, whichever is further from the learner's guess. If $f$ is the function that linearly interpolates the points $(x_0, g(x_0)), (x_1, g(x_1)), \dots$, then \[\int_{-\infty}^{\infty} |f'(x)|^q dx \le \sum_{i = 1}^{\infty} c^{i} \left(\frac{h}{c^{i}}\right)^{1+r} = h^{1+r} \sum_{i = 1}^{\infty} \frac{1}{c^{r i}} = \frac{h^{1+r}}{c^r-1}.\] If we choose $h \le (c^r-1)^{\frac{1}{1+r}}$, then $\int_{-\infty}^{\infty} |f'(x)|^q dx \le 1$. To conclude the proof, observe that in each round $i\ge 1$ the adversary
chooses $g(x_i)\in\{g(x_{i-1}),\,g(x_{i-1})+h\}$ so as to maximize the
absolute error $|\hat y_i-g(x_i)|$. In particular,
\[
|\hat y_i-g(x_i)| \;\ge\; \frac{h}{2}
\]
for every $i$, regardless of the learner's strategy. Hence the cumulative
loss satisfies
\[
L_{p,\mathbb{R}}(A,f,x)
\;\ge\;
\sum_{i=1}^\infty \Bigl(\frac{h}{2}\Bigr)^p
\;=\;\infty,
\]
for all $p\ge 1$. Since $f\in\mathcal{G}_q$ by the preceding calculation,
this shows that $\opt_{p,\mathbb{R}}(\mathcal{G}_q)=\infty$.

\end{proof}

The last proof shows that if the adversary chooses inputs that are far away (high $c$) from the inputs that the learner has seen so far, then the learner will have infinite error when the adversary plays optimally. The learner cannot even guarantee finite error on a single turn in this model, since the adversary can choose $c$ to be arbitrarily large.

There are multiple ways to address the issue of the difficulty in predicting $f(x)$ for inputs $x$ that are far away from all of the past inputs. We discuss some possibilities in the following sections: restricting the inputs, restricting the penalty function, and putting weights on the errors which depend on the distances to past inputs. Each of these possibilities assumes that $\mathcal{G}_q$ is still the family of functions.

On the other hand, we can also consider $\opt_{p,\mathbb{R}}(F)$ for other families of functions $F$. For example, let $F_L(n, \mathbb{R})$ denote the family of linear transformations from $\mathbb{R}^n$ to $\mathbb{R}$, i.e., $F_L(n, \mathbb{R})$ consists of the functions $f_v$ for which $f_v(x) = v \dot x$. Let $G_L(n, \mathbb{R},r)$ be the family of functions $g_v$ such that $g_v(x) = f_v(x)$ if $|f_v(x)| \le r$ and otherwise $g_v(x) = 0$.

\begin{thm}\label{thm:glnr}
For all $r>0$ and positive integers $n$, we have
\[
\opt_{p,\mathbb{R}}(G_L(n,\mathbb{R},r)) = n r^p.
\]
\end{thm}

\begin{proof}
    First, we show that $\opt_{p,\mathbb{R}}(G_L(n,\mathbb{R},r)) \le n r^p$.
Let the hidden function be $g$. Given an input $x_m$ not in the span of the
previous inputs $x_{i_1},\dots,x_{i_k}$ with nonzero outputs, the learner
guesses $0$. If $x_m=\sum_{j=1}^k c_{i_j}x_{i_j}$ and
$\bigl|\sum_{j=1}^k c_{i_j} g(x_{i_j})\bigr|\le r$, then the learner guesses
$\sum_{j=1}^k c_{i_j} g(x_{i_j})$. If $x_m=\sum_{j=1}^k c_{i_j}x_{i_j}$ and
$\bigl|\sum_{j=1}^k c_{i_j} g(x_{i_j})\bigr|> r$, then the learner guesses $0$.
Thus the learner is always correct whenever $x_m$ lies in the span of the
previous inputs with nonzero outputs. Whenever $x_m$ is not in that span, the
output is either $0$ or has absolute value at most $r$, so the absolute error
is at most $r$. Since there can be at most $n$ linearly independent inputs with
nonzero outputs, the learner makes at most $n$ such new-direction errors,
and each contributes at most $r^p$. Hence
$\opt_{p,\mathbb{R}}(G_L(n,\mathbb{R},r)) \le n r^p$.

Now we show that $\opt_{p,\mathbb{R}}(G_L(n,\mathbb{R},r)) \ge n r^p$.
Let $e_1,\dots,e_n$ denote the standard basis of $\mathbb{R}^n$.
In round $0$ the adversary plays $x_0=\mathbf{0}$.
For each $i=1,\dots,n$, in round $i$ the adversary plays $x_i=e_i$.
After the learner outputs $\hat y_i$, the adversary declares the true label to
be $+r$ or $-r$, whichever is farther from $\hat y_i$. This forces
$|\hat y_i - g(x_i)| \ge r$ for each $i=1,\dots,n$, hence incurs loss at least
$r^p$ in each of these $n$ rounds. Therefore
$\opt_{p,\mathbb{R}}(G_L(n,\mathbb{R},r)) \ge n r^p$.
\end{proof}

\section{Alternative scenarios for mitigating adversarial loss}\label{sec:3}

In this section, we introduce three possible learning scenarios that address the difficulty of predicting $f(x)$ for inputs $x$ that are far away from all of the past inputs. Each of these scenarios involves either modifying the original loss function, or adding input constraints to limit the adversary. The primary purpose of all of these scenarios is to limit the influence on the penalty function of inputs that are extremely far from previously observed values. 

\begin{description}
    \item[Scenario 1]
Restrictions on the Inputs 
\end{description}
We consider the learning scenario where we restrict the inputs $x_0, x_1, \dots \in \mathbb{R}^n$ so that for all $t \ge 1$, there must exist $i < t$ such that $d(x_i,x_t) \le 1$. Define \[L'_{p,\mathbb{R}}(A, \mathcal{F}) = \displaystyle \sup L'_{p,\mathbb{R}}(A, f, x),\] where the supremum is taken over all $f \in \mathcal{F}$ and $x \in  \cup_{m \in \mathbb{N}} \mathbb{R}^{m+1}$ such that \[\min_{i < t}d(x_t,x_i) \le 1\] for all $t > 0$. Furthermore, define the optimum $\opt'_{p,\mathbb{R}}(\mathcal{F}) = \displaystyle \inf_A L'_{p,\mathbb{R}}(A, \mathcal{F})$. 

This scenario applies a direct constraint to the adversary, by ensuring that each input that they choose after the first input is within a certain distance of the set of previous inputs. The loss function does not change from its original form, the only difference is the extra restriction on the adversary. 
Enforcing the restriction that \[\min_{i < t}d(x_t,x_i) \le 1\] for all $t > 0$ allows us to limit the adversary's ability to choose inputs that are arbitrarily far away from the previous inputs. In turn, this limits the adversary's ability to force arbitrarily large errors by the learner. 
\begin{description}
    \item[Scenario 2]
Free Guesses when Out of Bounds  
\end{description}

In this scenario, we remove the restrictions on the adversary's input choices, but we modify the loss function so that the learner is not penalized for any guess that occurs when a new input is at least $1$ away from all previous inputs.

For each $p \ge 1$ and $x = (x_0, \dots, x_m) \in \mathbb{R}^{m+1}$, define $x_{i_1}, \dots, x_{i_k}$ to be the maximal subsequence of $x_1, \dots, x_m$ such that \[\min_{j < i_t} d(x_{i_t},x_j) \le 1\] for all $t$ with $1 \le t \le k$. Let $L^{*}_{p,\mathbb{R}}(A, f, x) = \sum_{t = 1}^k |\hat{y}_{i_t}-f(x_{i_t})|^p$. Again note that the summation never includes the result of the first round. Define \[L^{*}_{p,\mathbb{R}}(A, \mathcal{F}) = \displaystyle \sup_{f \in \mathcal{F}, x \in  \cup_{m \in \mathbb{N}} \mathbb{R}^{m+1}} L^{*}_{p,\mathbb{R}}(A, f, x).\] We define the optimum $\opt^{*}_{p,\mathbb{R}}(\mathcal{F}) = \displaystyle \inf_A L^{*}_{p,\mathbb{R}}(A, \mathcal{F})$. 

 
This modification allows for selective penalization, where the algorithm is not blamed for errors on inputs that are far from all of the previous inputs. Unlike the first scenario, this allows the adversary to choose any inputs in the domain in any round, while ensuring that the learner is only penalized for errors on inputs that sufficiently close to previous inputs. From the definitions, it is easy to see that $\opt^{'}_{p,\mathbb{R}}(\mathcal{F}) \le \opt^{*}_{p,\mathbb{R}}(\mathcal{F})$ for all families of functions $\mathcal{F}$ with domain $\mathbb{R}$. 

\begin{lem}\label{lem:sc12}
For any family of functions $F$ and any $p\ge 1$,
\[
\opt'_{p,\mathbb{R}}(F)
\;\le\;
\opt^{*}_{p,\mathbb{R}}(F).
\]
\end{lem}

\begin{proof}
Every input sequence admissible in Scenario~1 (where
$\min_{i<t} d(x_t,x_i)\le 1$ for all $t>0$) is also admissible in
Scenario~2. On such sequences, the modified loss $L^*_{p,\mathbb{R}}$
coincides exactly with the original loss. Therefore any adversary strategy in
Scenario~1 can be simulated in Scenario~2 with the same incurred loss, which
implies the stated inequality.
\end{proof}


\begin{description}
    \item[Scenario 3]
Weighting of the Loss Function Based on Distance to Past Inputs
\end{description}
This scenario is parameterized by a nonnegative nonincreasing function $g: (0,\infty) \rightarrow [0,\infty)$, which modifies the standard p-norm function by incorporating a weight factor that diminishes the error penalty based on how far inputs are from previous data points. Throughout Scenario~3 we assume the input sequence $x=(x_0,\dots,x_m)$ has distinct entries, so that $\min_{0\le j<t} d(x_t,x_j)>0$ for every $t\ge1$. Given $g$, the adjusted loss function is defined as 

\[
L_{p,\mathbb{R}}^{g}(A, f, x)
\;=\;
\sum_{t=1}^{m}
g \bigl(\min_{0 \le j < t} d(x_{t}, x_{j})\bigr) \bigl|\hat{y}_t - f(x_t)\bigr|^p 
\] 

As with the other scenarios, define \[L^{g}_{p,\mathbb{R}}(A, \mathcal{F}) = \displaystyle \sup_{f \in \mathcal{F}, x \in  \cup_{m \in \mathbb{N}} \mathbb{R}^{m+1}} L^{g}_{p,\mathbb{R}}(A, f, x).\] We define the optimum $\opt^{g}_{p,\mathbb{R}}(\mathcal{F}) = \displaystyle \inf_A L^{g}_{p,\mathbb{R}}(A, \mathcal{F})$.

In this paper, we focus in particular on two choices of the weight function $g$. First, the \emph{identity weighting} corresponds to the choice
\[
g(z)=\frac{1}{z}.
\]
Substituting this into the definition of $L^{g}_{p,\mathbb{R}}$ gives
\[
L_{p,\mathbb{R}}^{\id}(A, f, x)
\;=\;
\sum_{t=1}^{m}
\frac{\bigl|\hat{y}_t - f(x_t)\bigr|^p}
{\min_{0 \le j < t} d(x_{t}, x_{j})}.
\] Second, we consider the \emph{exponential weighting}
\[
g(z)=e^{-cz}
\]
for a constant $c>0$, which yields
\[
L_{p,\mathbb{R}}^{\exp,c}(A, f, x)
\;=\;
\sum_{t=1}^{m}
\exp\!\Bigl(
  -\,c \,\min_{0 \le j < t} d(x_{t}, x_{j})
\Bigr)
\,\bigl|\hat{y}_t - f(x_t)\bigr|^p.
\]
A key conceptual feature of the weighted-loss formulation in Scenario 3 is that it
strictly generalizes Scenario~2. In particular, the ``free guess'' mechanism of
Scenario 2 can be realized as a special case of Scenario~3 by choosing a
discontinuous weight function
\[
g(z) \;=\; \mathbf{1}_{z \le 1}.
\]
Under this choice, any prediction made at distance $z>1$ incurs zero loss,
regardless of the predicted value, exactly matching the semantics of a free guess.
Consequently, Scenario 2 corresponds to a degenerate instance of Scenario 3.

We start with a lemma comparing Scenarios 2 and 3.

\begin{lem}\label{lem:comp23}
For all families of functions $F$ and all $p \ge 1$, we have
\[
\opt^{\id}_{p,\mathbb{R}}(F) \ge \opt^{*}_{p,\mathbb{R}}(F).
\]
\end{lem}

\begin{proof}
    In Scenario 2, the error weight is $1$ if $\min_{i < t}|x_t - x_i| \le 1$, and $0$ otherwise. In Scenario 3 with $\id$ (i.e., with $g(z)=1/z$), the error weight is at least $1$ whenever $\min_{i < t}|x_t - x_i| \le 1$, and positive otherwise. Thus, the adversary can use the same strategy for Scenario 3 as they would use for Scenario 2, and the total penalty for Scenario 3 will be at least the total penalty for Scenario 2.
\end{proof}

As a result of Lemma~\ref{lem:comp23}, we find some families of functions for which the learner cannot guarantee finite error in Scenario 3 with the identity function, when $p = 1$ or $q = 1$.

\begin{cor}
For all $p \ge 1$ and $q \ge 1$, we have $\opt^{\id}_{p,\mathbb{R}}(\mathcal{G}_1) = \infty$ and $\opt^{\id}_{1,\mathbb{R}}(\mathcal{G}_q) = \infty$.
\end{cor}

The last corollary can be generalized to functions $g$ such that $g(z) \ge \frac{1}{z}$ for all $z > 0$.

\begin{cor}
For all functions $g$ such that $g(z) \ge \frac{1}{z}$ for all $z > 0$ and for all real numbers $p \ge 1$ and $q \ge 1$, we have $\opt^{g}_{p,\mathbb{R}}(\mathcal{G}_1) = \infty$ and $\opt^{g}_{1,\mathbb{R}}(\mathcal{G}_q) = \infty$.
\end{cor}

Note that for any function $f \in \mathcal{F}_q$, we can extend $f$ to a function $g \in \mathcal{G}_q$ for which $f(x) = g(x)$ for all $x \in [0,1]$, $g(x) = f(0)$ for all $x < 0$, and $g(x) = f(1)$ for all $x > 1$. This leads to the following observation which we use for several results in this paper.

\begin{obs}\label{opt_comp}
For all $p, q \ge 1$, we have $\opt_{p,[0,1]}(\mathcal{F}_q) \le \opt'_{p,\mathbb{R}}(\mathcal{G}_q) \le \opt^{*}_{p,\mathbb{R}}(\mathcal{G}_q)$.
\end{obs}

\begin{proof}
All functions in $\mathcal{F}_q$ can be extended to functions in $\mathcal{G}_q$, and all inputs in $[0,1]$ are within distance $1$ of each other.
\end{proof}

Using the last observation, we obtain some immediate corollaries from the results of Kimber and Long \cite{klsf}.

\begin{cor}
For all $p \ge 1$ and $q \ge 1$, we have:
\begin{enumerate}
    \item $\opt'_{p,\mathbb{R}}(\mathcal{G}_1) = \infty$,
    \item $\opt'_{1,\mathbb{R}}(\mathcal{G}_q) = \infty$,
    \item $\opt^{*}_{p,\mathbb{R}}(\mathcal{G}_1) = \infty$,
    \item $\opt^{*}_{1,\mathbb{R}}(\mathcal{G}_q) = \infty$.
\end{enumerate} 
\end{cor}

\begin{proof}
This follows from Observation~\ref{opt_comp} since Kimber and Long proved that $\opt_{p,[0,1]}(\mathcal{F}_1) = \infty$ for all $p \ge 1$ and $\opt_{1,[0,1]}(\mathcal{F}_q) = \infty$ for all $q \ge 1$.
\end{proof}

In general, it will be convenient to compare different choices of weighting
functions in Scenario~3. For the next lemma we allow arbitrary strictly
positive functions $g,h:(0,\infty)\to(0,\infty)$.

\begin{lem}\label{lem:weight-compare}
Let $p>0$ and let $F$ be a family of real-valued functions on $\mathbb{R}$.
Let $g,h:(0,\infty)\to(0,\infty)$ and define
\[
C_{g,h} \;:=\; \sup_{z>0} \frac{h(z)}{g(z)} \;\in [0,\infty].
\]
If $C_{g,h}<\infty$, then
\[
\opt^{h}_{p,\mathbb{R}}(F) \;\le\;
C_{g,h}\,\opt^{g}_{p,\mathbb{R}}(F).
\]
\end{lem}

\begin{proof}
Fix an algorithm $A$ and write, for each trial $t\ge 1$ and input sequence
$x=(x_0,\dots,x_m)$, the distance
\[
\delta_t \;=\; \min_{0\le j < t} d(x_t,x_j).
\]
For any target function $f$ and any input sequence $x$ we have
\begin{align*}
L^{h}_{p,\mathbb{R}}(A,f,x)
&= \sum_{t=1}^m h(\delta_t) |\hat y_t-f(x_t)|^p \\
&= \sum_{t=1}^m \frac{h(\delta_t)}{g(\delta_t)}\,
         g(\delta_t) |\hat y_t-f(x_t)|^p  \\
&\le C_{g,h} \sum_{t=1}^m
         g(\delta_t) |\hat y_t-f(x_t)|^p \\
&= C_{g,h}\,L^{g}_{p,\mathbb{R}}(A,f,x),
\end{align*}
since $h(\delta_t)/g(\delta_t)\le C_{g,h}$ for every $t$.
Taking the supremum over $f\in F$ and $x$ yields
\[
L^{h}_{p,\mathbb{R}}(A,F) \;\le\;
C_{g,h}\,L^{g}_{p,\mathbb{R}}(A,F).
\]
Finally, taking the infimum over all algorithms $A$ gives
\[
\opt^{h}_{p,\mathbb{R}}(F)
= \inf_A L^{h}_{p,\mathbb{R}}(A,F)
\;\le\;
C_{g,h}\,\inf_A L^{g}_{p,\mathbb{R}}(A,F)
= C_{g,h}\,\opt^{g}_{p,\mathbb{R}}(F),
\]
as claimed.
\end{proof}

\begin{cor}
\label{cor:weight-monotone}
Let $g,h:(0,\infty)\to(0,\infty)$ be weight functions, and assume
$h(z)\le g(z)$ for all $z>0$.
Then for every function family $F$,
\[
\opt^{h}_{p,\R}(F)\le \opt^{g}_{p,\R}(F).
\]
\end{cor}

\begin{proof}
Since $h(z)\le g(z)$ pointwise, we have $\sup_{z>0}\frac{h(z)}{g(z)}\le 1$.
Apply Lemma~\ref{lem:weight-compare}.
\end{proof}

\begin{cor}\label{cor:exp_id}
For all $p>0$, $c>0$, and families of functions $F$, we have
\[
\opt^{\exp,c}_{p,\mathbb{R}}(F) \;\le\;
\frac{1}{c e}\,\opt^{\id}_{p,\mathbb{R}}(F).
\]
\end{cor}

\begin{proof}
As noted above,
$L^{\id}_{p,\mathbb{R}}$ is Scenario~3 with $g(z)= 1/z$ and
$L^{\exp,c}_{p,\mathbb{R}}$ is Scenario~3 with $h(z)=e^{-c z}$.
Applying Lemma~\ref{lem:weight-compare} with these choices gives
\[
\opt^{\exp,c}_{p,\mathbb{R}}(F)
\;\le\; C_{g,h}\,\opt^{\id}_{p,\mathbb{R}}(F),
\qquad
C_{g,h}=\sup_{z>0}\frac{h(z)}{g(z)}
=\sup_{z>0}\frac{e^{-cz}}{1/z}
=\sup_{z>0} z e^{-cz}.
\]
The function $\phi(z) = z e^{-c z}$ has derivative
$\phi'(z) = e^{-c z}(1-c z)$, so $\phi$ attains its unique maximum
at $z=1/c$, where
\[
\phi(1/c)
= \frac{1/c}{e}
= \frac{1}{c e}.
\]
Thus $C_{g,h}=1/(c e)$ and the stated inequality follows.
\end{proof}

\begin{cor}\label{cor:exp_orig}
For all $p>0$, $c>0$, and families of functions $F$, we have
\[
\opt^{\exp,c}_{p,\mathbb{R}}(F) \;\le\;
\opt_{p,\mathbb{R}}(F).
\]
\end{cor}

\begin{proof}
The original loss $L_{p,\mathbb{R}}(A,f,x)
=\sum_{t=1}^m|\hat y_t-f(x_t)|^p$ is the special case of
Scenario~3 with $g(z)\equiv 1$, so $\opt_{p,\mathbb{R}}(F)
=\opt^{g}_{p,\mathbb{R}}(F)$ for this choice of $g$.

Take $g(z)\equiv 1$ and $h(z)=e^{-c z}$ in Lemma~\ref{lem:weight-compare}. Then
\[
C_{g,h}
=\sup_{z>0} \frac{h(z)}{g(z)}
=\sup_{z>0} e^{-c z}
=1,
\]
since $e^{-c z}\le 1$ for all $z>0$ and tends to $1$ as $z\to 0^+$.
Thus
\[
\opt^{\exp,c}_{p,\mathbb{R}}(F)
=\opt^{h}_{p,\mathbb{R}}(F)
\le C_{g,h}\,\opt^{g}_{p,\mathbb{R}}(F)
= \opt_{p,\mathbb{R}}(F),
\]
as claimed.
\end{proof}

\begin{cor}\label{cor:g-vs-unweighted}
Let $p>0$, let $F$ be a family of functions, and let
$g:(0,\infty)\to(0,\infty)$ with
\[
\gamma := \sup_{z>0} g(z) > 0.
\]
Then
\[
\opt^{g}_{p,\mathbb{R}}(F) \;\le\; \gamma\,\opt_{p,\mathbb{R}}(F).
\]
In particular, if $g(z)\le 1$ for all $z>0$, then
$\opt^{g}_{p,\mathbb{R}}(F) \le \opt_{p,\mathbb{R}}(F)$.
\end{cor}

\begin{proof}
Apply Lemma~\ref{lem:weight-compare} with $h=g$ and $g_0(z)\equiv 1$.
Then
\[
C_{g_0,h} = \sup_{z>0} g(z) = \gamma,
\]
and $\opt^{g_0}_{p,\mathbb{R}}(F) = \opt_{p,\mathbb{R}}(F)$, which gives the claim.
\end{proof}

\section{Examples where scenarios 1 and 2 are equally hard for the learner}\label{sec:equal12}

In this section, we consider some families of functions $F$ where $\opt'_{p,\mathbb{R}}(F) = \opt^{*}_{p,\mathbb{R}}(F)$. Clearly we have this equality if $|F| = 1$, since the learner knows the hidden function from the beginning. In the next result, we see that this is also true when $|F| = 2.$

\begin{thm}
    If $|F| = 2$, then $\opt'_{p,\mathbb{R}}(F) = \opt^{*}_{p,\mathbb{R}}(F)$.
\end{thm}

\begin{proof}
    First, we start with an optimal learner strategy which works for both Scenario 1 and Scenario 2. In each round, suppose that the learner answers $\frac{f_1(x_i)+f_2(x_i)}{2}$ for every input $x_i$, unless the adversary says they are wrong. If the adversary says they are wrong, then it means that $f_1(x_i) \neq f_2(x_i)$, and the learner knows the hidden function once the adversary tells them the correct answer. This guarantees an error of at most $\frac{|f_1(x_i) - f_2(x_i)|}{2}$. Now, we show that this strategy is optimal, and it results in the same error for both scenarios 1 and 2, assuming that the adversary plays optimally.

    Let $S$ be the set of real numbers $x$ such that $x$ is within distance $1$ of some a real number $x’$ for which $f_1(x’)=f_2(x’)$. Let $T$ be the set of real numbers of the form $|f_1(x)-f_2(x)|$ for some $x \in S$. Let $m$ be the supremum of $T$. We claim that \[\opt'_{p,\mathbb{R}}(F) = \opt^{*}_{p,\mathbb{R}}(F) = \big(\frac{m}{2}\big)^p.\]

To prove this, we split into two cases. First, suppose that $T$ is unbounded. Then, the adversary can force error at least $k$ in the second round for any real number $k$. Indeed, since $T$ is unbounded, there exists some $x \in S$ which satisfies $|f_1(x)-f_2(x)| \ge k$. So, there is some $x’$ within distance $1$ of $x$ for which $f_1(x’)=f_2(x’)$. In the first round, the adversary gives the input $x’$. In the second round, the adversary gives the input $x$. Regardless of the learner’s answer, the adversary can guarantee error at least $\frac{|f_1(x)-f_2(x)|}{2}$. Thus, in this case we have \[\opt'_{p,\mathbb{R}}(F) = \opt^{*}_{p,\mathbb{R}}(F) = \infty.\]

Now, suppose that $T$ is bounded. Then, it has a supremum $m < \infty$. So, there exists some sequence of $x_1, x_2, \dots \in S$ such that for all $\epsilon > 0$ there exists $r$ for which $|f_1(x_i)-f_2(x_i)| > m-\epsilon$ for all $i > r$. The adversary uses the same strategy as in the last paragraph, forcing an error of at least $\frac{m-\epsilon}{2}$. Thus, in this case we have \[\opt'_{p,\mathbb{R}}(F) = \opt^{*}_{p,\mathbb{R}}(F) = \big(\frac{m}{2}\big)^p.\]
\end{proof}

Since the adversary strategy in Theorem~\ref{thm:glnr} uses only the zero vector and the standard basis vectors as inputs, we obtain the following result by definition of scenarios 1 and 2. 

\begin{thm}
    For all $r > 0$ and positive integers $n$, we have \[\opt'_{p,\mathbb{R}}(G_L(n, \mathbb{R},r)) = \opt^{*}_{p,\mathbb{R}}(G_L(n, \mathbb{R},r)) = n r^p.\]
\end{thm}

Next we prove that the bound in Observation~\ref{opt_comp} is sharp for $p = q = 2$. However for $q = \infty$ and $p \ge 2$ we show that the bound is not sharp. More specifically, we show the right side of $1 = \opt_{p,[0,1]}(\mathcal{F}_{\infty}) < \opt'_{p,\mathbb{R}}(\mathcal{G}_{\infty}) = \infty$ for $p \ge 2$ (the left side is from \cite{klsf}).

We introduce some terminology similar to \cite {klsf} and \cite{smooth2} that we use for the next result. For a function $f: \mathbb{R} \rightarrow \mathbb{R}$, we define $J[f] = \int_{-\infty}^{\infty} f'(x)^2 dx$, which is called the \emph{action} of $f$. Note that we use a slightly different definition of $J$ in this paper than the one in \cite{klsf}, the only difference being that we changed $[0, 1]$ to $\mathbb{R}$. 

Given a finite subset $S \subseteq \mathbb{R}^2$ with $S = \left\{(u_i, v_i): 1 \le i \le m\right\}$ and $u_1 < u_2 < \cdots < u_m$, we define $f_S: \mathbb{R} \rightarrow \mathbb{R}$ as follows. Let $f_{\emptyset}(x) = 0$ for all $x$, and for each nonempty $S$ let $f_S$ be the piecewise function defined by $f_S(x) = v_1$ for $x \le u_1$, $f_S(x) = v_i+\frac{(x-u_i)(v_{i+1}-v_i)}{u_{i+1}-u_i}$ for $x \in (u_i, u_{i+1}]$, and $f_S(x) = v_m$ for $x > u_m$. 

We use the $\linint$ learning algorithm which is defined in terms of $f_S$. Specifically we define $\linint(\emptyset, x_1) = 0$ and $\linint(((x_1, y_1), \dots, (x_{t-1},y_{t-1}),x_t) = f_{\left\{(x_1, y_1), \dots, (x_{t-1},y_{t-1}) \right\}}(x_t)$. 

Our proof uses the following two lemmas which are proved exactly the same way as Lemma 3 from \cite{smooth2} and Lemma 10 from \cite{klsf} respectively.

\begin{lem}\label{kl_jlem0}
Let $S \subseteq \mathbb{R}^2$ with $S = \left\{(u_i, v_i): 1 \le i \le m\right\}$ and $u_1 < u_2 < \cdots < u_m$. If $f$ is an absolutely continuous function with domain $\mathbb{R}$ such that $f(u_i) = v_i$ for all $i$, then $J[f] \ge J[f_S]$.
\end{lem}

\begin{proof}
See the proof of Lemma 3 in Appendix A of \cite{smooth2}, which still works when we replace $[0,1]$ with $\mathbb{R}$.
\end{proof}

\begin{lem}\label{kl_jlem}
Let $S \subseteq \mathbb{R}^2$ with $S = \left\{(u_i, v_i): 1 \le i \le m\right\}$ and $u_1 < u_2 < \cdots < u_m$. If $(x,y) \in \mathbb{R}^2$, then $J[f_{S \cup \left\{(x,y) \right\}}] \ge J[f_S] + \frac{(y-f_S(x))^2}{\min_i |x-u_i|}$. If there exists $1 \le j \le m$ such that $|x-u_j| = |x-u_{j+1}| = \min_i |x-u_i|$, then $J[f_{S \cup \left\{(x,y) \right\}}] = J[f_S] + \frac{2(y-f_S(x))^2}{\min_i |x-u_i|}$.
\end{lem}

\begin{proof}
This is proved the same way as Lemma 10 in \cite{klsf}.
\end{proof}

Now we explain how to obtain the value of $\opt'_{p,\mathbb{R}}(\mathcal{G}_q)$ for all $p \ge q \ge 2$ using the last two lemmas. 
We use $\linint'$, which is a modification of the $\linint$ algorithm. Suppose the adversary asks for the input $x$. If the learner does not know the value of the function at any input between $x-1$ and $x+1$, inclusive, then the learner guesses $0$. If the learner knows the value of the function at some input $a$ between $x-1$ and $x$, inclusive, where $a$ is as large as possible, but does not know the value of the function at any input between $x$ and $x+1$, inclusive, then the learner guesses $f(a)$. If the learner knows the value of the function at some input $b$ between $x$ and $x+1$, inclusive, where $b$ is as small as possible, but does not know the value of the function at any input between $x-1$ and $x$, inclusive, then the learner guesses $f(b)$. If the learner knows the value of the function at an input $a$ between $x-1$ and $x$, inclusive, and an input $b$ between $x$ and $x+1$, inclusive, where $a$ is as large as possible and $b$ is as small as possible, then the learner should guess the value according to the $\linint$ strategy. This value is equal to $f(a)+\frac{(x-a)(f(b)-f(a))}{b-a}$.

\begin{lem}
    For any $q>1$, target function $f \in \mathcal G_q$, integer $m \ge 1$, and sequence of inputs $x_0,\ldots,x_m \in \mathbb{R}$, $\linint'$ never produces an error $|\hat y_t-f(x_t)|>1$ on any trial $t \ge 1$ for which $\min_{i < t}|x_t - x_i| \le 1$.
\end{lem}

\begin{proof}
    If $f\in \mathcal G_q$, then by Lemma \ref{kl_jlem0}, for any $x_1<x_2$ with $x_2-x_1\leq 1$, we get \[1\geq\int_{-\infty}^{\infty}|f'(x)|^q dx\geq\int_{x_1}^{x_2}|f'(x)|^q dx\geq\left(\frac{|f(x_2)-f(x_1)|}{x_2-x_1}\right)^q(x_2-x_1)\geq|f(x_2)-f(x_1)|^q,\] so $|f(x_2)-f(x_1)|\leq 1$.
    
    Now, consider the error of the learner's $t$-th guess. We have two cases. There either exists $i<t$ such that $|x_t-x_i|\leq 1$ and the learner guesses $\hat{y}_t=f(x_i)$, or there exists $i,j<t$ such that $x_t-1\leq x_i\leq x_t\leq x_j\leq x_t+1$ and the learner guesses $\hat{y}_t=f(x_i)+\frac{(x_t-x_i)(f(x_j)-f(x_i))}{x_j-x_i}$. In the first case, $|\hat{y}_t-f(x_t)|=|f(x_i)-f(x_t)|\leq 1$, so the error is at most $1$. In the second case, we have $|f(x_t)-f(x_i)|\leq 1$ and $|f(x_t)-f(x_j)|\leq 1$. This implies
    \[\max(f(x_i),f(x_j))-1\leq f(x_t)\leq\min(f(x_i),f(x_j))+1.\]
    Since $\min(f(x_i),f(x_j))\leq\hat{y}_t\leq\max(f(x_i),f(x_j))$, this means $-1\leq \hat{y}_t-f(x_t)\leq 1$, so the error is at most $1$.
    
\end{proof}

\begin{thm}\label{thm:1pgeq}
If $p\geq q\geq 2$, then $\opt_{p,\R}'(\mathcal G_q)\leq\opt_{p,\R}^*(\mathcal G_q)\leq 1$.
\end{thm}
\begin{proof}
Suppose the learner uses $\linint$'. We will show that the $q$-action is always greater than or equal to the learner's error. By shifting the function, we can assume without loss of generality that the adversary asks for the value of $f(0)$.

If the learner does not know the value of the function at any input between $-1$ and $1$, then the learner will not increase the error.

If the learner knows the value of the function at an input between $-1$ and $0$ but at no input between $0$ and $1$, then let $a$ be the smallest positive real number such that the learner knows the value of $f(-a)$. Then, the learner guesses $f(0)=f(-a)$. If the learner's error increases by $x^p$, then the action increases by $\left(\frac xa\right)^q\geq x^q\geq x^p$ since $x\leq 1$. The case where the learner knows the value of the function at an input between $0$ and $1$ but at no input between $-1$ and $0$ is similar.

Suppose the learner knows the value of the function at an input between $-1$ and $0$ and at an input between $0$ and $1$. Suppose the learner knows $(-a,-ka)$ and $(b,kb)$ where $0\leq a,b\leq 1$ and $a$ and $b$ are minimal, and the learner guesses $f(0)=0$. The adversary then reveals $f(0)=c$, with $0\leq c\leq 1$. Then, the original action is $(a+b)k^q$ and the new action is $\frac{(c+ka)^q}{a^{q-1}}+\frac{|c-kb|^q}{b^{q-1}}$. The error is $c^p\leq c^q$. Assume without loss of generality $c\geq 0$. We need to show
\[\frac{(c+ka)^q}{a^{q-1}}+\frac{|c-kb|^q}{b^{q-1}}-(a+b)k^q\geq c^q\] when $0\leq a,b,c\leq 1$.

The derivative of the left hand side with respect to $a$ is equal to
\[\frac{(c+ka)^{q-1}(-cq+c+ka)-(ka)^q}{a^q}.\]

By weighted AM-GM,
\[\frac{1\cdot(ka)^q+(q-1)\cdot(c+ka)^q}q\geq\sqrt[q]{(ka)^{q\cdot 1}(c+ka)^{q\cdot(q-1)}}=ka(c+ka)^{q-1},\] which rearranges to $-(ka)^q+(c+ka)^{q-1}(-cq+c+ka)\leq 0$, so the expression $\frac{(c+ka)^q}{a^{q-1}}+\frac{|c-kb|^q}{b^{q-1}}-(a+b)k^q$ is minimized when $a=1$. Therefore,
\[\frac{(c+ka)^q}{a^{q-1}}+\frac{|c-kb|^q}{b^{q-1}}-(a+b)k^q\geq (c+k)^q+\frac{|c-kb|^q}{b^{q-1}}-(1+b)k^q.\]

Now, we need to show \[(c+k)^q+\frac{|c-kb|^q}{b^{q-1}}-(1+b)k^q-c^q\geq 0.\]

Case 1: $c\leq kb$

The expression is equal to $(c+k)^q+\frac{(kb-c)^q}{b^{q-1}}-(1+b)k^q-c^q$. The derivative with respect to $b$ is equal to \[\frac{(kb-c)^{q-1}(cq+kb-c)-(kb)^q}{b^q}=\frac{kbq(kb-c)^{q-1}-(q-1)(kb-c)^q-(kb)^q}{b^q}.\] By weighted AM-GM, \[\frac{1\cdot (kb)^q+(q-1)\cdot(kb-c)^q}q\geq\sqrt[q]{(kb)^{q\cdot 1}(kb-c)^{q\cdot(q-1)}}=kb(kb-c)^{q-1},\] so the derivative is always negative. Thus, the minimum occurs when $b=1$. The inequality now reduces to proving $(c+k)^q+(k-c)^q\geq 2k^q+c^q$ for all $0\leq c\leq k$. The derivative of $(c+k)^q+(k-c)^q-2k^q-c^q$ with respect to $c$ is equal to $q(c+k)^{q-1}-q(k-c)^{q-1}-qc^{q-1}\geq 0$ when $q\geq 2$, so this expression is minimized for $c\geq 0$ when $c=0$, which implies $(c+k)^q+(k-c)^q\geq 2k^q+c^q$.

Case 2: $c\geq kb$

The expression is equal to $(c+k)^q+\frac{(c-kb)^q}{b^{q-1}}-(1+b)k^q-c^q$. The derivative with respect to $b$ is equal to \[\frac{-(c-kb)^{q-1}(c(q-1)+bk)-b^qk^q}{b^q}<0,\] so the minimum occurs when $b=\min\left(\frac ck,1\right)$.

If $b=\frac ck$, then we need to prove $(c+k)^q-k^q-ck^{q-1}\geq c^q$ for all $0\leq c\leq k$. By Bernoulli's inequality, we have
$$\left(1+\frac ck\right)^q-1-\frac ck\geq1+\frac{qc}{k}-1-\frac ck=\frac{c(q-1)}{k}\geq\left(\frac ck\right)^q,$$ so multiplying by $k^q$ gives $(c+k)^q-k^q-ck^{q-1}\geq c^q$.

If $b=1$, then we need to prove $(c+k)^q+(c-k)^q\geq 2k^q+c^q$ for $c\geq k$. The derivative of $(c+k)^q+(c-k)^q-2k^q-c^q$ with respect to $c$ is $q(c+k)^{q-1}+q(c-k)^{q-1}-qc^{q-1}\geq 0$, so the minimum of $(c+k)^q+(c-k)^q-2k^q-c^q$ for $c\geq k$ occurs when $c=k$. Then, the expression is equal to $(2k)^q-3k^q>0$ since $q\geq 2$, so $(c+k)^q+(c-k)^q\geq 2k^q+c^q$.
\end{proof}

\begin{cor}
    If $p\geq q\geq 2$, then $\opt_{p,\R}'(\mathcal G_q)=\opt_{p,\R}^*(\mathcal G_q)=1$.
\end{cor}

\begin{proof}
    Since $\opt_{p,[0,1]}(\mathcal{F}_q) \le \opt_{p,\R}'(\mathcal G_q)\le\opt_{p,\R}^*(\mathcal G_q)=1$, this result follows from Theorem~\ref{thm:1pgeq} and the result from \cite{klsf} that $\opt_{p,[0,1]}(\mathcal{F}_q) = 1$ for all $p, q \ge 2$.
\end{proof}

In all of the examples so far, we have seen that $\opt_{p,[0,1]}(\mathcal{F}_q) = \opt'_{p,\mathbb{R}}(\mathcal{G}_q) = \opt^{*}_{p,\mathbb{R}}(\mathcal{G}_q)$. However, this is not always the case. Indeed, Kimber and Long \cite{klsf} proved that $\opt_{p,[0,1]}(\mathcal{F}_{q}) = 1$ for all $p,q \ge 2$. Here, we show that $\opt'_{p,\mathbb{R}}(\mathcal{G}_{q}) = \infty$ for all $q > p > 0$.

\begin{thm}\label{thm:inf_p<q}
    If $0<p<q$, then $\opt_{p,\R}'(\mathcal G_q)=\opt_{p,\R}^*(\mathcal G_q)=\infty$.
\end{thm}
\begin{proof}
    For any large positive integer $N$, let $\varepsilon=\frac{1}{N^{\frac{1}{q}}}$. In the first round, the adversary should reveal $g(0)=0$, then on round $i$ for $1\leq i\leq N$, the adversary sets $x_i=i$ and $g(i)=g(i-1)+\varepsilon$ or $g(i)=g(i-1)-\varepsilon$, whichever is further from the learner's guess. Note that $f_{(0,0),(1,g(1)),\dots,(N,g(N))}\in\mathcal G_q$ because the absolute value of its derivative is equal to $\varepsilon$ for all noninteger values of $x$ between $0$ and $N$. The learner's penalty is at least $N\varepsilon^p=N^{1-\frac{p}{q}}$, which can get arbitrarily large when $p<q$ for sufficiently large $N$. Therefore, the learner cannot guarantee any finite error.
\end{proof}

\begin{cor}
  For all $q > p\ge 2$, we have $\opt_{p,[0,1]}(\mathcal{F}_{q}) \neq \opt'_{p,\mathbb{R}}(\mathcal{G}_{q})$.
\end{cor}

\begin{proof}
This follows from Theorem~\ref{thm:inf_p<q} and the result from \cite{klsf} that $\opt_{p,[0,1]}(\mathcal{F}_{q}) = 1$ for all $p, q \ge 2$.
\end{proof}

\section{Examples where scenario 2 is harder for the learner than scenario 1}\label{sec:noteq}

We proved in the last section that $\opt'_{p,\mathbb{R}}(F) = \opt^{*}_{p,\mathbb{R}}(F)$ for all families of functions $F$ with $|F| = 2$. Given this result, it is natural to ask what is the smallest family of functions $F$ for which $\opt'_{p,\mathbb{R}}(F) < \opt^{*}_{p,\mathbb{R}}(F)$. We show that it is possible with $3$ functions.
Let $\varepsilon$ be sufficiently small. Define the family $F_{\varepsilon}$ of three functions

\begin{enumerate}
    \item $f_{\{(1,1),(2,0),(3,-\varepsilon),(4,0),(5,1)\}}$
    \item $f_{\{(1,1),(2,0),(3,0),(4,0),(5,-1)\}}$
    \item $f_{\{(1,-1),(2,0),(3,\varepsilon),(4,0),(5,-1)\}}$
\end{enumerate}

We claim that $\opt'_{p,\R}(F_{\varepsilon})<\opt^*_{p,\R}(F_{\varepsilon})$ for $\varepsilon$ sufficiently small.

\begin{thm}
    For all $p \ge 1$ and $\varepsilon > 0$, we have $\opt'_{p,\R}(F_{\varepsilon})\leq 1+\varepsilon^p$.
\end{thm}
\begin{proof}
The learner uses the strategy of guessing $0$, except when a previous answer from the adversary implies the value of the current output.

If the adversary ever asks for an input strictly between $2$ and $4$, then the learner has error at most $\varepsilon$, and the learner now knows the correct function, so the learner will not make any more mistakes. All previous inputs must be either all at most $2$ or all at least $4$. Assume without loss of generality all previous inputs are at most $2$. Then, the first two functions are identical, so once the learner makes one mistake, the learner knows the correct output for all inputs at most $2$. This mistake has error at most $1$. Thus, we have $\opt'_{p,\R}(F)\leq 1+\varepsilon^p$.
\end{proof}

\begin{thm}
    For all $p \ge 1$ and $\varepsilon > 0$, we have $\opt^*_{p,\R}(F_{\varepsilon}) \ge 1+2^{-p}$.
\end{thm}
\begin{proof}
The adversary first asks for the input $1$. Suppose the learner responds with the output $y$. If $|1+y|^p \ge 1+2^{-p}$, then the adversary picks the third function and responds with $f(1)=-1$. Otherwise, if $|1+y|^p < 1+2^{-p}$, the adversary responds with $f(1)=1$, so the learner's error is $|1-y|$. Now, the adversary should ask for the input $4$ and must respond with $f(4)=0$. After, the adversary asks for the input $5$ and responds with $f(5)=1$ or $f(5)=-1$, whichever is further from the learner's guess. The error is at least $1$, so the learner's penalty is at least $1+|1-y|^p$. Suppose that this penalty is at most $1+2^{-p}$. Then, $|1-y|\leq \frac{1}{2}$, so $y\geq \frac{1}{2}$. This means $|1+y|^p\geq\big( \frac{3}{2} \big)^p \ge 1+2^{-p}$, which is impossible since we assumed $|1+y|^p < 1+2^{-p}$. Therefore, the learner's penalty is always at least $1+2^{-p}$, so $\opt^*_{p,\R}(F) \ge 1+2^{-p}$.
\end{proof}

Thus, $\opt'_{p,\R}(F_{\varepsilon})<\opt^*_{p,\R}(F_{\varepsilon})$ for all $\varepsilon < \frac{1}{2}$. Since there exist families for which scenario 2 is harder for the learner than scenario 1, it is natural to investigate whether there is a general upper bound on the ratio \[\frac{\opt^*_{p,\R}(F)}{\opt'_{p,\R}(F)}.\] We show that there is no constant upper bound.

Let $k$ be a positive integer. For each $i$, let $S_{i,\varepsilon}$ be a set of ordered pairs which contains $(4j,0)$, $(4j+1,i\varepsilon)$, $(4j+2,0)$, and $\left(4j+3,(-1)^{\left\lfloor\frac{i}{2^j}\right\rfloor}\right)$ for all $0\leq j\leq k-1$. For $n=2^k$, define the family $F_{n,\varepsilon}$ given by the functions $f_{S_{0,\varepsilon}}$, $f_{S_{1,\varepsilon}}$, \dots, $f_{S_{n-1,\varepsilon}}$.

\begin{thm}
    For all $p$, we have $\opt'_{p,\R}(F_{n,\varepsilon})\leq 1+(n\varepsilon)^p$.
\end{thm}
\begin{proof}
The learner uses the strategy of guessing $0$, except when a previous answer from the adversary implies the value of the current output.

If the adversary ever asks for an input $x$ such that $4j<x<4j+2$ for some $0\leq j\leq k-1$, then the learner has error at most $n\varepsilon$, and the learner now knows the correct function, so the learner will not make any more mistakes. All previous inputs must satisfy $4j+2\leq x\leq 4j+4$ for some fixed $0\leq j\leq k-1$, $x\leq 0$, or $x\geq 4k-2$. Then, there are only two distinct functions restricted to the domain $4j+2\leq x\leq 4j+4$, depending on the parity of $\left\lfloor\frac{i}{2^{j}}\right\rfloor$. In the domain $x\leq 0$, the function is constant and equal to $0$. In the domain $x\geq 4k-2$, there are also only two possible functions depending on the parity of $\left\lfloor\frac{i}{2^{k-1}}\right\rfloor$. Therefore, after the learner makes one mistake, the learner knows the correct value of the function at all other inputs in this restricted domain, so the learner will make no further mistakes. The error of the mistake is at most $1$. Therefore, the total penalty is at most $1+(n\varepsilon)^p$.
\end{proof}

\begin{thm}
    For all $p$, we have $\opt^*_{p,\R}(F_{n,\varepsilon})\geq k$.
\end{thm}
\begin{proof}
The adversary asks for the inputs at $x=4j+2$, then $x=4j+3$ for $0\leq j\leq k-1$ in order. The adversary always reveals $f(4j+2)=0$, and whenever the learner guesses a value for $x=4j+3$, the adversary reveals $f(4j+3)=1$ or $f(4j+3)=-1$, whichever is further from the learner's guess. This set of function values is consistent because $f_i(4j+3)=1$ implies the $2^j$ bit of $i$ is $0$, and $f_i(4j+3)=-1$ implies the $2^j$ bit of $i$ is $1$. For any set of values for $f(4j+3)$ for $0\leq j\leq k-1$, there is always a unique $i$ which has the correct bits at each position. Then, the learner's penalty must be at least $k$.
\end{proof}

Therefore, we have $$\lim_{\varepsilon\to 0}\frac{\opt^*_{p,\R}(F_{n,\varepsilon})}{\opt'_{p,\R}(F_{n,\varepsilon})}=k=\Theta(\log n).$$

\begin{prob}
    For each $n$, what is the supremum of $\frac{\opt^*_{p,\R}(F)}{\opt'_{p,\R}(F)}$ over all families of functions $F$ containing $n$ functions?
\end{prob}

In the next result, we obtain an upper bound on the answer to the last problem.

\begin{thm}
If $F$ is a family of $n$ functions, then $\frac{\opt^*_{p,\R}(F)}{\opt'_{p,\R}(F)}\leq n-1$.
\end{thm}
\begin{proof}
Consider the strategy where at each input, the learner guesses the average of the smallest and largest possible value of any function in $F$ which is consistent with all previous inputs and outputs. If the learner is incorrect at some input, then at least one new function is eliminated, so the learner can only be incorrect at most $n-1$ times. In scenario 2, there exists an adversary strategy such that the total error is at least $\opt^*_{p,\R}(F)$, so in at least one round, the learner has an error of at least $\frac{\opt^*_{p,\R}(F)}{n-1}$. Suppose the learner guesses $f(x_1)=y$ in this round, and the adversary reveals $f(x_1)=y_1$ in this round and $f(x_2)=y_2$ in a previous round where $|x_1-x_2|\leq1$. There must exist two functions $f_1$ and $f_2$ in $F$ which satisfy $f_1(x_2)=f_2(x_2)=y_2$, $\frac{f_1(x_1)+f_2(x_1)}2=y$, and $f_1(x_1)\leq y_1\leq f_2(x_1)$. This implies $\frac{f_2(x_1)-f_1(x_1)}2\geq|y_1-y|$, so $\left(\frac{f_2(x_1)-f_1(x_1)}2\right)^p\geq\frac{\opt^*_{p,\R}(F)}{n-1}$.

Now, in scenario $1$, let the adversary first reveal $f(x_2)=y_2$, then ask for the value of $f(x_1)$. Then, the adversary reveals either $f_1(x_1)$ or $f_2(x_1)$, whichever is further from the learner's guess. The error is at least $\left(\frac{f_2(x_1)-f_1(x_1)}2\right)^p$, which means $\opt'_{p,\R}(F)\geq \left(\frac{f_2(x_1)-f_1(x_1)}2\right)^p\geq\frac{\opt^*_{p,\R}(F)}{n-1}$, so $\frac{\opt^*_{p,\R}(F)}{\opt'_{p,\R}(F)}\leq n-1$.
\end{proof}

\begin{thm}\label{thm:nratio}
    For each $n$, $p$, there exists a family of $n$ functions $G_{n,\varepsilon}$ for every $\varepsilon>0$ such that $\frac{\opt^*_{p,\R}(G_{n,\varepsilon})}{\opt'_{p,\R}(G_{n,\varepsilon})}\geq\frac{n-1}{\left(\frac{1+(n-1)^{\frac1p}}2\right)^p}\cdot\frac{1}{1+(n\varepsilon)^p}$.
\end{thm}
\begin{proof}
For each $i$, let $T_{i,\varepsilon}$ be a set of ordered pairs which contains $(2j,0)$ for all $0\leq j\leq 2n$, $(4j-3,i\varepsilon)$ for all $1\leq j\leq n$, $(4j-1,-1)$ for all $j\neq i$ such that $1\leq j\leq n$, and $(4i-1,1)$. Define the family $G_{n,\varepsilon}$ given by the functions $f_{T_{1,\varepsilon}}$, $f_{T_{2,\varepsilon}}$, \dots, $f_{T_{n,\varepsilon}}$.

In scenario $1$, let the learner uses the strategy of guessing $0$, except when previous answers from the adversary imply the value of the current output. Note that whenever the adversary asks for an input strictly between $4j-4$ and $4j-2$ for some $1\leq j\leq n$, the learner will have error at most $n\varepsilon$ and not make any more mistakes. Thus, in scenario $1$, the learner can only make mistakes between $4j-2$ and $4j$, inclusive, for some fixed $1\leq j\leq n$. In this interval, there are only two distinct functions, so the learner will make at most one mistake. This mistake has error at most $1$. Therefore, the learner's total penalty is at most $1+(n\varepsilon)^p$, so $\opt'_{p,\R}(G_{n,\varepsilon})\leq 1+(n\varepsilon)^p$.

In scenario $2$, fix the constant $c=\frac{n-1}{\left(\frac{1+(n-1)^{\frac1p}}2\right)^p}$. The adversary asks for the inputs at $x=4j-2$, then $x=4j-1$ for $1\leq j\leq n-1$ in order. The adversary always reveals $f(4j-2)=0$. If the learners guess differs from $1$ by at least $c^{\frac{1}{p}}$, then the adversary reveals $f(4j-1)=1$ and responds to all other inputs with the value of $f_{T_{j,\varepsilon}}$. The penalty in this case is at least $c$. Otherwise, the adversary responds $f(4j-1)=-1$. In this case, the learner must guess a number which is at least $1-c^{\frac{1}{p}}=\frac{1-(n-1)^{\frac1p}}{1+(n-1)^{\frac1p}}$. The error in each of the $n-1$ guesses is at least $\frac{2}{1+(n-1)^{\frac1p}}$, so the total penalty is at least $\frac{n-1}{\left(\frac{1+(n-1)^{\frac1p}}{2}\right)^p}=c$. Therefore, the adversary guarantees a penalty of at least $c$. This means $\opt^*_{p,\R}(G_{n,\varepsilon})\geq c$.
\end{proof}
\begin{cor}
    \[\lim_{p\to\infty}\sup\Bigl\{\frac{\opt^*_{p,\mathbb{R}}(F)}{\opt'_{p,\mathbb{R}}(F)}:\ |F|=n\Bigr\}\geq\sqrt{n-1}\]
\end{cor}

\begin{proof}
Fix $n\ge 2$. For each $p\ge 1$, apply Theorem~\ref{thm:nratio} with
\[
\varepsilon_p := \frac{1}{n}\,p^{-1/p}.
\]
Then $(n\varepsilon_p)^p = p^{-1}$, and hence
\[
\frac{1}{1+(n\varepsilon_p)^p}=\frac{1}{1+p^{-1}}\longrightarrow 1
\qquad (p\to\infty).
\]

It therefore remains to compute
\[
\lim_{p\to\infty}
\frac{n-1}{\left(\frac{1+(n-1)^{1/p}}{2}\right)^p}.
\]
Set $a_p := (n-1)^{1/p}=\exp\!\bigl(\tfrac{\log(n-1)}{p}\bigr)$, and note that
$a_p\to 1$. We rewrite
\[
\left(\frac{1+a_p}{2}\right)^p
=\exp\!\left(p\log\!\left(\frac{1+a_p}{2}\right)\right)
=\exp\!\left(p\,\phi\!\left(\frac{\log(n-1)}{p}\right)\right),
\]
where
\[
\phi(t):=\log\!\left(\frac{1+e^{t}}{2}\right).
\]
Since $\phi$ is differentiable at $t=0$, we have
\[
\lim_{t\to 0}\frac{\phi(t)-\phi(0)}{t}=\phi'(0).
\]
Now $\phi(0)=\log(1)=0$, and a direct computation gives
\[
\phi'(t)=\frac{e^t}{1+e^t}\quad\Rightarrow\quad \phi'(0)=\frac12.
\]
Therefore,
\[
\lim_{t\to 0}\frac{\phi(t)}{t}=\frac12.
\]
With $t_p:=\frac{\log(n-1)}{p}\to 0$, this implies
\[
\lim_{p\to\infty} p\,\phi(t_p)
=\lim_{p\to\infty}\bigl(p t_p\bigr)\cdot\frac{\phi(t_p)}{t_p}
=\log(n-1)\cdot\frac12
=\frac12\log(n-1).
\]
Exponentiating yields
\[
\left(\frac{1+(n-1)^{1/p}}{2}\right)^p
\longrightarrow \exp\!\left(\frac12\log(n-1)\right)=\sqrt{n-1}.
\]
Consequently,
\[
\frac{n-1}{\left(\frac{1+(n-1)^{1/p}}{2}\right)^p}
\longrightarrow \frac{n-1}{\sqrt{n-1}}=\sqrt{n-1}.
\]

Finally, applying Theorem~\ref{thm:nratio} to the family
$G_{n,\varepsilon_p}$ and taking the supremum over all families $F$ with $|F|=n$
gives
\[
\lim_{p\to\infty}\sup\Bigl\{
\frac{\opt^*_{p,\mathbb{R}}(F)}{\opt'_{p,\mathbb{R}}(F)}:\ |F|=n
\Bigr\}\ge \sqrt{n-1},
\]
as claimed.
\end{proof}

\section{Radius parameter and scaling for Scenarios 1 and 2}\label{sec:rd12}
In earlier sections, we focused on Scenarios 1 and 2 with a radius of $1$. Here, we consider Scenarios 1 and 2 with an arbitrary radius $R > 0$.

\begin{defn}\label{def:radiusR}
Fix $R>0$.

\smallskip
\noindent
(Scenario~1 with radius $R$.)
An input sequence $x=(x_0,\dots,x_m)\in\mathbb{R}^{m+1}$ is \emph{$R$-admissible}
if for every $t\in\{1,\dots,m\}$ there exists $i<t$ such that
$|x_t-x_i|\le R$.
Define the corresponding loss and optimum by
\[
L^{\prime,R}_{p,\mathbb{R}}(A,f,x)
:=\sum_{t=1}^{m}|\hat y_t-f(x_t)|^p
\quad\text{(with the supremum restricted to $R$-admissible sequences)},
\]
\[
L^{\prime,R}_{p,\mathbb{R}}(A,F)
:=\sup_{f\in F}\ \sup_{\substack{x\ \text{$R$-admissible}}}
L^{\prime,R}_{p,\mathbb{R}}(A,f,x),
\qquad
\opt^{\prime,R}_{p,\mathbb{R}}(F):=\inf_A L^{\prime,R}_{p,\mathbb{R}}(A,F).
\]

\smallskip
\noindent
(Scenario~2 with radius $R$.)
Given an arbitrary sequence $x=(x_0,\dots,x_m)\in\mathbb{R}^{m+1}$, let
\[
\mathcal{T}_R(x):=\{t\in\{1,\dots,m\}:\min_{0\le j<t}|x_t-x_j|\le R\}
\]
be the set of \emph{$R$-penalized rounds}. Define
\[
L^{*,R}_{p,\mathbb{R}}(A,f,x)
:=\sum_{t\in\mathcal{T}_R(x)}|\hat y_t-f(x_t)|^p,
\qquad
L^{*,R}_{p,\mathbb{R}}(A,F)
:=\sup_{f\in F}\ \sup_{x} L^{*,R}_{p,\mathbb{R}}(A,f,x),
\]
and
\[
\opt^{*,R}_{p,\mathbb{R}}(F):=\inf_A L^{*,R}_{p,\mathbb{R}}(A,F).
\]
\end{defn}

\begin{lem}\label{lem:radius-mono}
Fix $p>0$ and a function family $F$ on $\mathbb{R}$.

\begin{enumerate}
\item If $0<R_1\le R_2$, then
\[
\opt^{\prime,R_1}_{p,\mathbb{R}}(F)\ \le\ \opt^{\prime,R_2}_{p,\mathbb{R}}(F).
\]
\item If $0<R_1\le R_2$, then
\[
\opt^{*,R_1}_{p,\mathbb{R}}(F)\ \le\ \opt^{*,R_2}_{p,\mathbb{R}}(F).
\]
\end{enumerate}
\end{lem}

\begin{proof}
(1) Every $R_1$-admissible input sequence is also $R_2$-admissible, so for every learner $A$,
$L^{\prime,R_1}_{p,\mathbb{R}}(A,F)\le L^{\prime,R_2}_{p,\mathbb{R}}(A,F)$.
Taking infima over $A$ yields the claim.

\smallskip
\noindent
(2) If $R_1\le R_2$, then for every input sequence $x$ we have
$\mathcal{T}_{R_1}(x)\subseteq \mathcal{T}_{R_2}(x)$. Hence for every learner $A$,
target $f$, and sequence $x$,
$L^{*,R_1}_{p,\mathbb{R}}(A,f,x)\le L^{*,R_2}_{p,\mathbb{R}}(A,f,x)$.
Taking suprema over $f$ and $x$, and then infima over $A$, gives the stated
inequality.
\end{proof}

\begin{lem}\label{lem:scale-Gq}
Fix $q>1$ and $R>0$. Define the scaling operator on functions
$f:\mathbb{R}\to\mathbb{R}$ by
\[
(\mathsf{S}_R f)(x):=R^{-(q-1)/q}\,f(Rx).
\]
Then:
\begin{enumerate}
\item If $f$ is absolutely continuous, then $\mathsf{S}_R f$ is absolutely
continuous and for a.e.\ $x$ we have
\[
(\mathsf{S}_R f)'(x)=R^{1/q}\,f'(Rx).
\]
\item Membership in $\mathcal{G}_q$ is preserved:
\[
f\in\mathcal{G}_q\quad\Longleftrightarrow\quad \mathsf{S}_R f\in\mathcal{G}_q.
\]
More precisely,
\[
\int_{\mathbb{R}} |(\mathsf{S}_R f)'(x)|^q\,dx
=
\int_{\mathbb{R}} |f'(x)|^q\,dx.
\]
\end{enumerate}
\end{lem}

\begin{proof}
The composition $x\mapsto f(Rx)$ is absolutely continuous, and multiplication
by the constant $R^{-(q-1)/q}$ preserves absolute continuity. Hence
$\mathsf{S}_R f$ is absolutely continuous, and the chain rule yields
\[
(\mathsf{S}_R f)'(x)=R^{1/q}f'(Rx)
\quad\text{for almost every }x.
\]

For the $L^q$-derivative integral, apply the preceding identity and change
variables $u=Rx$:
\[
\int_{\mathbb{R}} |(\mathsf{S}_R f)'(x)|^q\,dx
=
\int_{\mathbb{R}} \bigl|R^{1/q} f'(Rx)\bigr|^q\,dx
=
\int_{\mathbb{R}} R\,|f'(Rx)|^q\,dx
=
\int_{\mathbb{R}} |f'(u)|^q\,du.
\]
Thus $\int |f'|^q\le 1$ iff $\int |(\mathsf{S}_R f)'|^q\le 1$, which is exactly
$f\in\mathcal{G}_q$ iff $\mathsf{S}_R f\in\mathcal{G}_q$.
\end{proof}

\begin{thm}\label{prop:reduce-R-to-1}
Fix $q>1$, $p>0$, and $R>0$.
Then for Scenarios~1--2 (Definition~\ref{def:radiusR}) we have
\[
\opt^{\prime,R}_{p,\mathbb{R}}(\mathcal{G}_q)
=
R^{(q-1)p/q}\,\opt^{\prime,1}_{p,\mathbb{R}}(\mathcal{G}_q),
\qquad
\opt^{*,R}_{p,\mathbb{R}}(\mathcal{G}_q)
=
R^{(q-1)p/q}\,\opt^{*,1}_{p,\mathbb{R}}(\mathcal{G}_q),
\]
with the natural conventions that $c\cdot\infty=\infty$ for every $c>0$.
In particular, the choice of threshold $1$ is without loss of generality up to
a deterministic scaling factor.
\end{thm}

\begin{proof}
We prove the Scenario~1 identity; the Scenario~2 identity is analogous.

Let $A$ be any learner for Scenario~1 with radius $1$. We build a learner
$A_R$ for Scenario~1 with radius $R$ as follows. On round $t$, when the adversary
presents $x_t\in\mathbb{R}$, define the rescaled input $z_t:=x_t/R$ and feed
$z_t$ to $A$. If $A$ outputs $\widehat w_t$ (its prediction for the value of the
unknown function at $z_t$), then $A_R$ outputs
\[
\hat y_t:=R^{(q-1)/q}\,\widehat w_t.
\]

Now fix any target $f\in\mathcal{G}_q$. Define the rescaled target
\[
\widetilde f:=\mathsf{S}_R f,
\qquad\text{i.e.}\qquad
\widetilde f(z)=R^{-(q-1)/q} f(Rz).
\]
By Lemma~\ref{lem:scale-Gq}, we have $\widetilde f\in\mathcal{G}_q$.

Also, if the original input sequence $(x_t)$ is $R$-admissible, then the
rescaled sequence $(z_t)$ is $1$-admissible, because
$|z_t-z_i|=|x_t-x_i|/R\le 1$ whenever $|x_t-x_i|\le R$.

Finally, for each round $t$ we have
\[
f(x_t)=f(Rz_t)=R^{(q-1)/q}\,\widetilde f(z_t),
\]
and therefore
\[
|\hat y_t-f(x_t)|
=
\left|R^{(q-1)/q}\widehat w_t - R^{(q-1)/q}\widetilde f(z_t)\right|
=
R^{(q-1)/q}\,|\widehat w_t-\widetilde f(z_t)|.
\]
Raising to the $p$th power and summing (over the same set of rounds, since in
Scenario~1 every round $t\ge 1$ is counted) gives
\[
L^{\prime,R}_{p,\mathbb{R}}(A_R,f,x)
=
R^{(q-1)p/q}\,L^{\prime,1}_{p,\mathbb{R}}(A,\widetilde f,z).
\]
Taking suprema over $f\in\mathcal{G}_q$ and $R$-admissible $x$, we obtain
\[
L^{\prime,R}_{p,\mathbb{R}}(A_R,\mathcal{G}_q)
\le
R^{(q-1)p/q}\,L^{\prime,1}_{p,\mathbb{R}}(A,\mathcal{G}_q),
\]
because $(\widetilde f,z)$ ranges over a subset of the pairs allowed in the
radius-$1$ game. Infimizing over all $A$ yields
\[
\opt^{\prime,R}_{p,\mathbb{R}}(\mathcal{G}_q)
\le
R^{(q-1)p/q}\,\opt^{\prime,1}_{p,\mathbb{R}}(\mathcal{G}_q).
\]
Equivalently,
\[
R^{-(q-1)p/q}\,\opt^{\prime,R}_{p,\mathbb{R}}(\mathcal{G}_q)
\le
\opt^{\prime,1}_{p,\mathbb{R}}(\mathcal{G}_q).
\]

We now prove the reverse inequality.
Let $B$ be an arbitrary learner for Scenario~1 with radius $R$.
We construct from $B$ a learner $B_{1/R}$ for Scenario~1 with radius $1$.

On round $t$, when the adversary presents an input $z_t\in\mathbb{R}$,
the learner $B_{1/R}$ feeds the expanded input
\[
x_t:=R z_t
\]
to $B$. If $B$ outputs a prediction $\widehat y_t$ for the value of the
unknown function at $x_t$, then $B_{1/R}$ outputs
\[
\widehat w_t:=R^{-(q-1)/q}\,\widehat y_t
\]
as its prediction at $z_t$.

Fix any target function $\widetilde f\in\mathcal{G}_q$ for the radius-$1$ game.
Define
\[
f(x):=R^{(q-1)/q}\,\widetilde f(x/R),
\]
so that $\widetilde f=\mathsf{S}_R f$. By Lemma~\ref{lem:scale-Gq}, we have
$f\in\mathcal{G}_q$.

If the input sequence $(z_t)$ is $1$-admissible, then the expanded sequence
$(x_t)$ is $R$-admissible, since
\[
|x_t-x_i|=R|z_t-z_i|\le R
\quad\text{whenever}\quad
|z_t-z_i|\le 1.
\]
Thus $(x_t)$ is a valid adversarial sequence for $B$.

Moreover, for each round $t$,
\[
\widetilde f(z_t)
=
R^{-(q-1)/q} f(x_t),
\]
and therefore
\[
|\widehat w_t-\widetilde f(z_t)|
=
R^{-(q-1)/q}\,|\widehat y_t-f(x_t)|.
\]
Raising to the $p$th power and summing over $t\ge 1$ gives
\[
L^{\prime,1}_{p,\mathbb{R}}(B_{1/R},\widetilde f,z)
=
R^{-(q-1)p/q}\,
L^{\prime,R}_{p,\mathbb{R}}(B,f,x).
\]

Taking the supremum over all $\widetilde f\in\mathcal{G}_q$ and all
$1$-admissible sequences $z$, we obtain
\[
L^{\prime,1}_{p,\mathbb{R}}(B_{1/R},\mathcal{G}_q)
\le
R^{-(q-1)p/q}\,
L^{\prime,R}_{p,\mathbb{R}}(B,\mathcal{G}_q).
\]
Finally, taking the infimum over all radius-$R$ learners $B$ yields
\[
\opt^{\prime,1}_{p,\mathbb{R}}(\mathcal{G}_q)
\le
R^{-(q-1)p/q}\,
\opt^{\prime,R}_{p,\mathbb{R}}(\mathcal{G}_q),
\]
as claimed.

Combining the two inequalities proves the identity.
The Scenario~2 identity is proved identically, observing that
$t\in\mathcal{T}_R(x)$ iff $t\in\mathcal{T}_1(z)$ under the scaling
$z_t=x_t/R$.
\end{proof}

\section{Scenario 3}\label{sec:sc3}

First, we show that the learner can guarantee finite error in Scenario 3 with the identity function, when $p = q = 2$.

\begin{thm}\label{thm:s322}
    $\opt^{\id}_{2,\mathbb{R}}(\mathcal{G}_2) = 1$
\end{thm}

\begin{proof}
The lower bound $\opt^{\id}_{p,\mathbb{R}}(\mathcal{G}_2) \ge 1$ follows from Lemma~\ref{lem:comp23}, since $\opt^{*}_{2,\mathbb{R}}(\mathcal{G}_2) = 1$. Thus, it suffices to prove that $\opt^{\id}_{2,\mathbb{R}}(\mathcal{G}_2) \le 1$.

The learner uses $\linint$. Let $g \in \mathcal{G}_2$ be a target function and $x_1, x_2, \dots$ be distinct reals. By Lemma \ref{kl_jlem}, the action of the hypothesis of $\linint$ increases by at least \[\frac{(\hat{y}_t-g(x_t))^2}{\min_{i < t}|x_t - x_i|}\] for each $t > 0$. After $t = 0$, the hypothesis has action $0$. By Lemma \ref{kl_jlem0}, the hypothesis always has action at most $J[g]$, which is at most $1$ since $g \in \mathcal{G}_2$. Thus \[\sum_{t \ge 1} \frac{(\hat{y}_t-g(x_t))^2}{\min_{i < t}|x_t - x_i|} \le 1,\] so we have $\opt^{\id}_{2,\mathbb{R}}(\mathcal{G}_2) \le 1$.
\end{proof}

The following theorem shows that Corollary~\ref{cor:exp_id} is sharp.

\begin{thm}
    For all positive $f: \mathbb{R}^{\ge 0} \rightarrow \mathbb{R}^{\ge 0}$, let $\gamma = \sup_{x > 0} x f(x)$. Then, we have \[\opt^{f}_{2,\mathbb{R}}(\mathcal{G}_2) = \gamma.\]
\end{thm}

\begin{proof}
    If the learner uses $\linint$, then they guarantee that \[\opt^{f}_{2,\mathbb{R}}(\mathcal{G}_2) \le \opt^{\id}_{2,\mathbb{R}}(\mathcal{G}_2) \sup_{x > 0} x f(x) = \sup_{x > 0} x f(x) = \gamma\] by Theorem~\ref{thm:s322} and Lemma~\ref{lem:weight-compare}. For the corresponding lower bound, fix $\varepsilon > 0$ and let $x_{\varepsilon}$ be a positive real number such that \[x_{\varepsilon} f(x_{\varepsilon}) > \gamma - \varepsilon.\] 
    The adversary uses the strategy of asking the input $0$ in the first round, saying the correct output is $0$, asking the input $x_{\varepsilon}$ in the second round, and saying the correct output is $\pm \sqrt{x_{\varepsilon}}$, whichever is farther from the learner's guess. 
    It is simple to check that the resulting function is in the family $\mathcal{G}_2$, since \[x_{\varepsilon}\left( \sqrt{\frac{1}{x_{\varepsilon}}} \right)^2 = 1.\] 
    Moreover, the total penalty will be at least \[f(x_{\varepsilon})\left(\sqrt{x_{\varepsilon}}\right)^2 > \gamma - \varepsilon.\] 
    The adversary can pick $\varepsilon$ arbitrarily close to $0$, so we have\[\opt^{f}_{2,\mathbb{R}}(\mathcal{G}_2) = \gamma.\] 
\end{proof}

\begin{cor}
    For all $c > 0$, we have \[\opt^{\exp, c}_{2,\mathbb{R}}(\mathcal{G}_2) = \frac{1}{c e}.\]
\end{cor}

\begin{proof}
   This follows since \[\sup_{x > 0} x e^{-cx} = \frac{1}{c e}.\]
\end{proof}

In the next result, we obtain a general divergence criterion for Scenario 3 with the families $\mathcal{G}_q$.

\begin{prop}\label{prop:slowdecay}
Let $g:(0,\infty)\to(0,\infty)$ be nonincreasing and satisfy
\[
\sum_{i=1}^\infty g(c^i)=\infty
\quad\text{for some } c>1.
\]
Then for all $p\ge 1$ and $q>1$,
\[
\opt^g_{p,\mathbb{R}}(\mathcal{G}_q)=\infty.
\]
\end{prop}

\begin{proof}
Fix $c>1$ and define $x_0=0$ and $x_i=x_{i-1}+c^i$ for $i\ge 1$.
Let $h>0$ be chosen so that
\[
h^q \sum_{i=1}^\infty c^{-(q-1)i} \le 1,
\quad\text{e.g.}\quad
h \le (c^{q-1}-1)^{1/q}.
\]
Define values $g(x_0)=0$ and for each $i\ge 1$ choose
\[
g(x_i)\in\{g(x_{i-1}),\,g(x_{i-1})+h\}
\]
to maximize $|\hat y_i-g(x_i)|$, and let $f$ be the linear interpolant of the
points $(x_i,g(x_i))$.

On the segment $[x_{i-1},x_i]$ the slope has magnitude either $0$ or $h/c^i$,
so
\[
\int_{-\infty}^\infty |f'(x)|^q\,dx
\le
\sum_{i=1}^\infty c^i\Bigl(\frac{h}{c^i}\Bigr)^q
=
h^q \sum_{i=1}^\infty c^{-(q-1)i}
\le 1,
\]
hence $f\in\mathcal{G}_q$.

Since the two candidate labels differ by $h$, the adversary guarantees
$|\hat y_i-f(x_i)|\ge h/2$ for every $i\ge 1$.
Moreover, for each $i\ge 1$ we have
\[
\delta_i := \min_{0\le j<i} d(x_i,x_j) = d(x_i,x_{i-1}) = c^i,
\]
so the weight in Scenario~3 is $g(\delta_i)=g(c^i)$. Therefore
\[
L^g_{p,\mathbb{R}}(A,f,x)
\;\ge\;
\sum_{i=1}^\infty g(c^i)\Bigl(\frac{h}{2}\Bigr)^p.
\]
Since $\sum_{i=1}^\infty g(c^i)=\infty$, the right-hand side diverges, proving
the claim.
\end{proof}

The remaining results focus on the family $G_L(n, \mathbb{R},r)$.

\begin{thm}
    For all $p, r > 0$ and positive integers $n$, we have $\opt^{\id}_{p,\mathbb{R}}(G_L(n, \mathbb{R},r)) = \infty$.
\end{thm}

\begin{proof}
    In the first round, the adversary gives the zero vector as input and says that the correct output is $0$. In the second round, the adversary can force arbitrarily large error. Indeed, the adversary fixes some real number $\epsilon > 0$ and picks any non-zero vector within distance $\epsilon$ of the zero vector as the second input. After the learner guesses the output, the adversary says that the correct answer is $\pm r$, whichever is farther from the learner's answer. Thus, the adversary can force the learner to be off by at least $r$ from the correct answer, so the adversary adds at least $\frac{1}{\epsilon}$ to the total penalty in the second round. Since $\epsilon$ can be arbitrarily close to $0$, the adversary can force arbitrarily large error in the second round. 
\end{proof}

The following theorem shows that Corollary~\ref{cor:exp_orig} is sharp.

\begin{thm}
    For all $c, p, r > 0$ and positive integers $n$, we have $\opt^{\exp,c}_{p,\mathbb{R}}(G_L(n, \mathbb{R},r)) = n r^p$.
\end{thm}

\begin{proof}
    By Corollary~\ref{cor:exp_orig}, we have $\opt^{\exp,c}_{p,\mathbb{R}}(G_L(n, \mathbb{R},r)) \le n r^p$. For the lower bound, consider the following adversary strategy. The adversary fixes a real number $\epsilon > 0$. In round $0$, the adversary gives the all-zero vector as the input $x_0$. In round $i$ for each $1 \le i \le n$, the adversary gives the input $x_i = \epsilon e_i$. After the learner guesses, the adversary says that the correct answer is $\pm r$, whichever is farther from the learner's guess. This guarantees an error of at least $r$ in the $i^{\text{th}}$ round for each $1 < i \le n+1$, so we have \[\opt^{\exp,c}_{p,\mathbb{R}}(G_L(n, \mathbb{R},r)) \ge \frac{n r^p}{e^{c \epsilon}}.\] Since $\epsilon$ can be chosen to be arbitrarily close to $0$, we have $\opt^{\exp,c}_{p,\mathbb{R}}(G_L(n, \mathbb{R},r)) \ge n r^p$. Thus, we have $\opt^{\exp,c}_{p,\mathbb{R}}(G_L(n, \mathbb{R},r)) = n r^p$.
\end{proof}

\section{Online learning of multivariable functions}\label{sec:Gqd-opt}

In this section we define the multivariable slice class $\mathcal{G}_{q,d}$
(the unbounded-domain analogue of the coordinate-slice class $\mathcal{F}_{q,d}$
from \cite{GZ}) and we derive consequences \emph{only} for the minimax quantities
$\opt'_{p,\mathbb{R}^d}$, $\opt^*_{p,\mathbb{R}^d}$, and $\opt^g_{p,\mathbb{R}^d}$
in Scenarios~1--3.

We first record a monotonicity principle in the dimension $d$, proved by an
isometric embedding argument and an explicit inclusion of the function classes.
We then prove that for $d\ge 2$ the slice constraint is too permissive on
$\mathbb{R}^d$: even under the strong locality built into Scenarios~1--2, the
adversary can force infinite loss for every $p>0$ and every $q>1$.  The proof is
constructive and verifies membership in $\mathcal{G}_{q,2}$ by direct
computation on every coordinate slice.

\subsection{Definition}

\begin{defn}\label{def:Gqd}
Fix $d\in\mathbb{Z}_+$ and $q\in(1,\infty]\cup\{\infty\}$.
Let $\mathcal{G}_{q,d}$ denote the family of functions $f:\mathbb{R}^d\to\mathbb{R}$
such that the following holds.

For each coordinate index $i\in\{1,\dots,d\}$ and each choice of the remaining
coordinates $u=(u_1,\dots,u_{d-1})\in\mathbb{R}^{d-1}$, define the one-variable
slice
\[
g_{i,u}(t)
:=
f(u_1,\dots,u_{i-1},t,u_i,\dots,u_{d-1}),
\qquad t\in\mathbb{R}.
\]
Then $g_{i,u}$ is required to belong to $\mathcal{G}_q$ (in the sense of
Section~\ref{sec:2}), i.e.\ $g_{i,u}$ is absolutely continuous and satisfies
\[
\int_{-\infty}^{\infty} |g'_{i,u}(t)|^q\,dt \le 1
\quad\text{(or }\esssup_{t\in\mathbb{R}}|g'_{i,u}(t)|\le 1\text{ if }q=\infty).
\]
\end{defn}

\subsection{Comparison with the bounded-domain slice classes $\mathcal{F}_{q,d}$}

The coordinate-slice classes $\mathcal{F}_{q,d}$ introduced in \cite{GZ} impose the
same \emph{one-dimensional} $L^q$-derivative constraint as in
Definition~\ref{def:Gqd}, but on the compact domain $[0,1]^d$ rather than on
$\mathbb{R}^d$.  Concretely, $f\in\mathcal{F}_{q,d}$ means $f:[0,1]^d\to\mathbb{R}$
and every coordinate slice $t\mapsto f(u_1,\dots,u_{i-1},t,u_i,\dots,u_{d-1})$
lies in the single-variable class $\mathcal{F}_q$ (i.e.\ is absolutely continuous
on $[0,1]$ with $\int_0^1 |g'(t)|^q\,dt\le 1$, or $\|g'\|_\infty\le 1$ when $q=\infty$).

\smallskip
\noindent
\textbf{What is known for $\mathcal{F}_{q,d}$.}
The multi-variable results in \cite{GZ} show that slice constraints on a compact
domain can still yield a meaningful (dimension-dependent) learnability picture.
First, \cite[Prop.~3.1]{GZ} proves a general lower bound relating the $d$-variable
problem to the one-variable problem:
\[
\opt_p(\mathcal{F}_{q,d}) \;\ge\; d^p\,\opt_p(\mathcal{F}_q).
\]
Second, when $q=\infty$ (coordinatewise $1$-Lipschitz on $[0,1]^d$),
\cite[Thm.~1.4]{GZ} identifies a sharp finiteness threshold in the exponent:
\[
\opt_p(\mathcal{F}_{\infty,d})<\infty \ \text{ for } p>d,
\qquad
\opt_p(\mathcal{F}_{\infty,d})=\infty \ \text{ for } 0<p<d.
\]
In particular, since $\mathcal{F}_{\infty,d}\subseteq \mathcal{F}_{q,d}$ for every
$q\ge 1$, it follows that $\opt_p(\mathcal{F}_{q,d})=\infty$ for all $q\ge 1$ and
all $0<p<d$ (see \cite[Cor.~3.4]{GZ}).

\smallskip
\noindent
\textbf{Contrast with $\mathcal{G}_{q,d}$ on $\mathbb{R}^d$.}
The class $\mathcal{G}_{q,d}$ is the unbounded-domain analogue: we require the same
coordinate-slice $L^q$-derivative control, but with integrals over $\mathbb{R}$.
Our results in this section show that for $d\ge 2$ this analogue is dramatically
more permissive in the worst-case online setting considered here: even under the
strong locality constraints built into Scenarios~1--2, the minimax $p$-loss is
infinite for every $p>0$ and every $q>1$ (Theorem~\ref{thm:Gqd-inf-all}), and the
weighted variant in Scenario~3 diverges whenever $g(1)>0$
(Proposition~\ref{prop:Gqd-sc3-inf}).  Informally, the key geometric difference is
that on $\mathbb{R}^d$ an adversary can place infinitely many disjoint,
well-separated ``local bumps'' (each consuming only bounded $q$-action on every
coordinate slice) along an admissible input sequence with fixed separation,
forcing a constant error on every round; this mechanism has no direct analogue on
the compact cube $[0,1]^d$.

\subsection{Monotonicity in the dimension}

\begin{prop}\label{prop:monotone-d}
Fix integers $1\le d_0\le d_1$, reals $p>0$ and $q>1$, and (for Scenario~3) a
weight $g:(0,\infty)\to(0,\infty)$.

\begin{enumerate}
\item (Scenario~1) \;
$\opt'_{p,\mathbb{R}^{d_1}}(\mathcal{G}_{q,d_1})
\;\ge\;
\opt'_{p,\mathbb{R}^{d_0}}(\mathcal{G}_{q,d_0})$.

\item (Scenario~2) \;
$\opt^*_{p,\mathbb{R}^{d_1}}(\mathcal{G}_{q,d_1})
\;\ge\;
\opt^*_{p,\mathbb{R}^{d_0}}(\mathcal{G}_{q,d_0})$.

\item (Scenario~3) \;
$\opt^{g}_{p,\mathbb{R}^{d_1}}(\mathcal{G}_{q,d_1})
\;\ge\;
\opt^{g}_{p,\mathbb{R}^{d_0}}(\mathcal{G}_{q,d_0})$.
\end{enumerate}
\end{prop}

\begin{proof}
Let $\iota:\mathbb{R}^{d_0}\to\mathbb{R}^{d_1}$ be the map
\[
\iota(z_1,\dots,z_{d_0}) := (z_1,\dots,z_{d_0},0,\dots,0).
\]
Then for all $z,z'\in\mathbb{R}^{d_0}$,
\[
\|\iota(z)-\iota(z')\|_2=\|z-z'\|_2,
\]
so $\iota$ is an isometric embedding and in particular preserves all quantities
$\delta_t=\min_{j<t} d(x_t,x_j)$ appearing in Scenarios~1--3.

\smallskip
\noindent
\textbf{Step 1: class inclusion $\mathcal{G}_{q,d_0}\hookrightarrow \mathcal{G}_{q,d_1}$.}
Let $h\in\mathcal{G}_{q,d_0}$.  Define $f:\mathbb{R}^{d_1}\to\mathbb{R}$ by
\[
f(x_1,\dots,x_{d_1}) := h(x_1,\dots,x_{d_0}).
\]
We prove that $f\in\mathcal{G}_{q,d_1}$ by checking the defining condition in
Definition~\ref{def:Gqd}.

Fix $i\in\{1,\dots,d_1\}$ and fix the remaining coordinates
$u\in\mathbb{R}^{d_1-1}$.  Define the corresponding slice
\[
g_{i,u}(t)
=
f(u_1,\dots,u_{i-1},t,u_i,\dots,u_{d_1-1}),
\qquad t\in\mathbb{R}.
\]

\emph{Case 1: $i\le d_0$.}
Define $\widetilde u\in\mathbb{R}^{d_0-1}$ by deleting from
$(u_1,\dots,u_{d_0})$ the coordinate occupying position $i$ (so that inserting
$t$ back into the $i$th coordinate yields $(u_1,\dots,u_{i-1},t,u_i,\dots,u_{d_0-1})$).
Then, by the definition of $f$,
\[
g_{i,u}(t)
=
h(u_1,\dots,u_{i-1},t,u_i,\dots,u_{d_0-1})
=
:=
\widetilde g_{i,\widetilde u}(t),
\]
where $\widetilde g_{i,\widetilde u}$ is exactly the $i$th coordinate slice of
$h$ with the other coordinates fixed to $\widetilde u$.
Since $h\in\mathcal{G}_{q,d_0}$, Definition~\ref{def:Gqd} implies that
$\widetilde g_{i,\widetilde u}\in\mathcal{G}_q$.  Therefore
$g_{i,u}\in\mathcal{G}_q$.

\emph{Case 2: $i>d_0$.}
Then $f$ does not depend on the $i$th coordinate, so $g_{i,u}$ is constant.
A constant function on $\mathbb{R}$ is absolutely continuous and has derivative
$0$ almost everywhere, hence belongs to $\mathcal{G}_q$.

In both cases we have shown $g_{i,u}\in\mathcal{G}_q$, and since $i$ and $u$
were arbitrary, Definition~\ref{def:Gqd} implies $f\in\mathcal{G}_{q,d_1}$.

\smallskip
\noindent
\textbf{Step 2: reduction of $d_1$-dimensional learners to $d_0$-dimensional learners.}
Fix a learner $A_{d_1}$ for the $d_1$-dimensional game (Scenario~1,~2, or~3).
Define a learner $A_{d_0}$ for the $d_0$-dimensional game by the following
simulation: on input $z_t\in\mathbb{R}^{d_0}$, feed $\iota(z_t)\in\mathbb{R}^{d_1}$
to $A_{d_1}$ and output the same prediction.  When the adversary reveals the
label $h(z_t)$, pass the label $f(\iota(z_t))$ to $A_{d_1}$; these are equal by
the definition of $f$.

Because $\iota$ is an isometry, a sequence $(z_t)$ satisfies the Scenario~1
constraint in $\mathbb{R}^{d_0}$ if and only if $(\iota(z_t))$ satisfies it in
$\mathbb{R}^{d_1}$.  Likewise, for every trial $t$, the quantities
$\delta_t=\min_{j<t}\|z_t-z_j\|_2$ and
$\widetilde\delta_t=\min_{j<t}\|\iota(z_t)-\iota(z_j)\|_2$ are equal, so the set
of penalized rounds in Scenario~2 and the weights $g(\delta_t)$ in Scenario~3
coincide under the embedding.  Consequently, for each target function $h$ and
each input sequence $(z_t)$, the loss incurred by $A_{d_0}$ against $h$ equals
the loss incurred by $A_{d_1}$ against $f$ on the embedded sequence.

Taking the supremum over targets and sequences (respecting the scenario’s
admissibility/penalty rule) and then taking the infimum over learners yields the
three stated inequalities.
\end{proof}

\subsection{Scenarios~1--2 have infinite error for $d\ge 2$, $p>0, q > 1$}

The next theorem shows that the slice-based notion of smoothness, which is perfectly adequate on bounded domains, breaks down learning on $\mathbb{R}^d$ in dimensions $d \ge 2$. 
Even when the adversary is forced to move only unit distance at each step, and even when the learner is penalized only locally, the adversary can force a constant error on every round while still respecting the slice constraint defining $\mathcal{G}_{q,d}$. 
In this sense, the class $\mathcal{G}_{q,d}$ is fundamentally non-learnable on unbounded domains under any of the locality models considered in this paper.

\begin{thm}\label{thm:Gqd-inf-all}
Fix $d\ge 2$, $q>1$, and $p>0$.  Then
\[
\opt'_{p,\mathbb{R}^d}(\mathcal{G}_{q,d})
=
\opt^*_{p,\mathbb{R}^d}(\mathcal{G}_{q,d})
=
\infty.
\]
\end{thm}

\begin{proof}
By Proposition~\ref{prop:monotone-d}, it suffices to prove the claim for $d=2$.

\smallskip
\noindent
\textbf{Step 1: a unit-step input sequence.}
Fix $\eta\in(0,1/10)$ and set $\alpha:=\sqrt{1-\eta^2}$.
Define
\[
x_0=(0,0),
\qquad
x_t=(t\alpha,\,t\eta)\quad\text{for }t\ge 1.
\]
Then $\|x_t-x_{t-1}\|_2=1$ for all $t\ge 1$.  Hence this input sequence is
admissible in Scenario~1, and in Scenario~2 every round $t\ge 1$ is counted
because $\min_{0\le j<t}\|x_t-x_j\|_2=1$.

\smallskip
\noindent
\textbf{Step 2: explicit tents and exact $L^q$-derivative computations.}
Let $L:=\alpha/4$ and $b:=\eta/4$.  Define
\[
u(s):=
\begin{cases}
1-\frac{|s|}{L}, & |s|\le L,\\
0, & |s|>L,
\end{cases}
\qquad
\phi(t):=
\begin{cases}
1-\frac{|t|}{b}, & |t|\le b,\\
0, & |t|>b.
\end{cases}
\]
Then $u$ and $\phi$ are absolutely continuous, $u(0)=\phi(0)=1$, and for a.e.\ $s$
and $t$ we have $|u'(s)|=1/L$ on $(-L,L)\setminus\{0\}$ and $|\phi'(t)|=1/b$ on
$(-b,b)\setminus\{0\}$. Therefore
\[
\int_{\mathbb{R}} |u'(s)|^q\,ds = 2L\cdot (1/L)^q = 2L^{1-q},
\qquad
\int_{\mathbb{R}} |\phi'(t)|^q\,dt = 2b\cdot (1/b)^q = 2b^{1-q}.
\]
Define
\[
a := \min\{(2L^{1-q})^{-1/q},\ (2b^{1-q})^{-1/q}\}.
\]
Let $v_0\equiv 0$ and $v_1:=a\,u$.  Then $v_0(0)=0$, $v_1(0)=a$, and
\[
\int_{\mathbb{R}} |v_1'(s)|^q\,ds
=
a^q\int_{\mathbb{R}} |u'(s)|^q\,ds
\le 1,
\qquad
\int_{\mathbb{R}} |v_0'(s)|^q\,ds =0\le 1,
\]
and also $|v_1(s)|\le a$ for all $s$.

\smallskip
\noindent
\textbf{Step 3: disjoint rectangles and an online-consistent definition of $f$.}
For each $t\ge 1$ define
\[
J_t:=[t\alpha-L,\;t\alpha+L],
\qquad
K_t:=[t\eta-b,\;t\eta+b],
\qquad
\phi_t(y):=\phi(y-t\eta).
\]
Since $2L=\alpha/2<\alpha$, the sets $J_t$ are pairwise disjoint.
Since $2b=\eta/2<\eta$, the sets $K_t$ are pairwise disjoint.
Hence $R_t:=J_t\times K_t$ are pairwise disjoint rectangles.

After observing the learner's prediction $\hat y_t$ at input $x_t$, the adversary
chooses $\sigma_t\in\{0,1\}$ and defines $f$ on $R_t$ by
\[
f(x,y):=\phi_t(y)\,v_{\sigma_t}(x-t\alpha),
\qquad (x,y)\in R_t,
\]
and defines $f(x,y):=0$ on $\mathbb{R}^2\setminus\bigcup_{t\ge 1} R_t$.
Because the sets $R_t$ are disjoint, this defines a single well-defined function
$f:\mathbb{R}^2\to\mathbb{R}$.

\smallskip
\noindent
\textbf{Step 4: forcing a constant error each round.}
At $x_t=(t\alpha,t\eta)$, we have $\phi_t(t\eta)=\phi(0)=1$ and $x-t\alpha=0$, so
\[
f(x_t)=v_{\sigma_t}(0)\in\{0,a\}.
\]
The adversary chooses $\sigma_t$ so that the value $f(x_t)$ is whichever of
$0$ and $a$ that is farther from $\hat y_t$.  Consequently
$|\hat y_t-f(x_t)|\ge a/2$ for all $t\ge 1$, and therefore the cumulative
$p$-loss in either Scenario~1 or Scenario~2 is at least
$\sum_{t\ge 1}(a/2)^p=\infty$.

\smallskip
\noindent
\textbf{Step 5: verifying $f\in\mathcal{G}_{q,2}$.}
We check the defining condition in Definition~\ref{def:Gqd} for both coordinate
directions.

\emph{(Slices in the $x$-direction.)}
Fix $y\in\mathbb{R}$.  Because the intervals $K_t$ are pairwise disjoint, there
is at most one index $t$ such that $y\in K_t$.
If there is no such $t$, then $f(\cdot,y)\equiv 0$ and hence
$\int_{\mathbb{R}}|\partial_x f(x,y)|^q\,dx=0\le 1$.
If $y\in K_t$ for some $t$, then for all $x\in\mathbb{R}$,
\[
f(x,y)=\phi_t(y)\,v_{\sigma_t}(x-t\alpha),
\]
which is absolutely continuous in $x$, with a.e.\ derivative
$\partial_x f(x,y)=\phi_t(y)\,v_{\sigma_t}'(x-t\alpha)$.  Therefore
\[
\int_{\mathbb{R}} |\partial_x f(x,y)|^q\,dx
=
|\phi_t(y)|^q \int_{\mathbb{R}} |v_{\sigma_t}'(s)|^q\,ds.
\]
Since $\sigma_t\in\{0,1\}$, we have either $v_{\sigma_t}=v_0$ with
$\int_{\mathbb{R}} |v_0'(s)|^q\,ds=0$, or $v_{\sigma_t}=v_1$ with
$\int_{\mathbb{R}} |v_1'(s)|^q\,ds\le 1$ by construction.
In both cases,
\[
\int_{\mathbb{R}} |v_{\sigma_t}'(s)|^q\,ds \le 1.
\]
Because $|\phi_t(y)|\le 1$, it follows that
\[
\int_{\mathbb{R}} |\partial_x f(x,y)|^q\,dx \le 1.
\]

\emph{(Slices in the $y$-direction.)}
Fix $x\in\mathbb{R}$.  Because the intervals $J_t$ are pairwise disjoint, there
is at most one index $t$ such that $x\in J_t$.
If there is no such $t$, then $f(x,\cdot)\equiv 0$ and hence
$\int_{\mathbb{R}}|\partial_y f(x,y)|^q\,dy=0\le 1$.
If $x\in J_t$ for some $t$, then for all $y\in\mathbb{R}$,
\[
f(x,y)=v_{\sigma_t}(x-t\alpha)\,\phi_t(y),
\]
which is absolutely continuous in $y$, with a.e.\ derivative
$\partial_y f(x,y)=v_{\sigma_t}(x-t\alpha)\,\phi_t'(y)$.  Thus
\[
\int_{\mathbb{R}} |\partial_y f(x,y)|^q\,dy
=
|v_{\sigma_t}(x-t\alpha)|^q \int_{\mathbb{R}} |\phi'(s)|^q\,ds
\le
a^q \cdot 2b^{1-q}
\le
1,
\]
because $|v_{\sigma_t}|\le a$ and $a\le (2b^{1-q})^{-1/q}$ by definition.

In all cases the relevant one-variable slice belongs to $\mathcal{G}_q$, and
therefore $f\in\mathcal{G}_{q,2}$.  This completes the proof for $d=2$, and the
general case $d\ge 2$ follows from Proposition~\ref{prop:monotone-d}.
\end{proof}

\subsection{Scenario~3: divergence whenever $g(1)>0$ and $d \ge 2$}

\begin{prop}\label{prop:Gqd-sc3-inf}
Fix $d\ge 2$, $q>1$, and $p>0$.  Let $g:(0,\infty)\to(0,\infty)$ be nonincreasing
and satisfy $g(1)>0$.  Then
\[
\opt^{g}_{p,\mathbb{R}^d}(\mathcal{G}_{q,d})=\infty.
\]
\end{prop}

\begin{proof}
Use the same input sequence as in Theorem~\ref{thm:Gqd-inf-all}.
For every $t\ge 1$,
\[
\delta_t=\min_{0\le j<t}\|x_t-x_j\|_2=\|x_t-x_{t-1}\|_2=1,
\]
so every weighted term is multiplied by $g(1)$.
The adversary forces $|\hat y_t-f(x_t)|\ge a/2$ for all $t\ge 1$, hence
\[
L^{g}_{p,\mathbb{R}^d}(A,f,x)
\;\ge\;
\sum_{t=1}^{\infty} g(1)\,(a/2)^p
\;=\;\infty.
\]
Taking the infimum over learners yields $\opt^{g}_{p,\mathbb{R}^d}(\mathcal{G}_{q,d})=\infty$.
\end{proof}

\section{Conclusion}\label{sec:conc}

We initiated a systematic study of the mistake-bound model for learning
real-valued smooth functions on an \emph{unbounded} domain, focusing on the classes $\mathcal{G}_q$ of absolutely continuous functions
$f:\mathbb{R}\to\mathbb{R}$ with $\int_{\mathbb{R}} |f'(x)|^q\,dx\le 1$.
Our first observation is that the most direct extension of the classical
$[0,1]$ model is ill-posed on $\mathbb{R}$: for every $p\ge 1$ and $q>1$ the
adversary can force infinite $p$-loss, even in two rounds, by placing the next
query arbitrarily far from all previous inputs.
This motivates the central theme of the paper: to obtain meaningful guarantees
on $\mathbb{R}$ one must either constrain the admissible input sequences or
modify the penalty to discount ``exploration'' far from previously queried
points.

We proposed and analyzed three such mitigations.
Scenario~1 restricts the adversary by requiring each new input to lie within
distance $1$ of some previously queried point.
Scenario~2 keeps unrestricted inputs but grants ``free guesses'' on rounds in
which the new input is farther than $1$ from the past.
Scenario~3 introduces a distance-dependent weighting $g$ in the loss, replacing
each error term by $g(\min_{j<t} d(x_t,x_j))|\hat y_t-f(x_t)|^p$.
These scenarios preserve the adversarial, worst-case nature of the model while
capturing different notions of locality and exploration on an unbounded domain.

Several general comparisons emerge.
We established basic dominance relations between weightings
(Lemma~\ref{lem:weight-compare}) and used them to relate exponential and
identity-type weightings (Corollary~\ref{cor:exp_id}), as well as weighted and
unweighted losses (Corollary~\ref{cor:g-vs-unweighted}).
For Scenarios~1 and~2 we identified a broad regime in which the unbounded-domain
behavior matches the classical bounded-domain picture: when $p\ge q\ge 2$,
the modified linear interpolation algorithm $\linint'$ achieves the sharp value
$\opt'_{p,\mathbb{R}}(\mathcal{G}_q)=\opt^*_{p,\mathbb{R}}(\mathcal{G}_q)=1$.
In contrast, when $0<p<q$ the adversary can force arbitrarily large loss even
under Scenarios~1 and~2, yielding $\opt'_{p,\mathbb{R}}(\mathcal{G}_q)
=\opt^*_{p,\mathbb{R}}(\mathcal{G}_q)=\infty$.
We also showed that Scenarios~1 and~2 need not be equivalent: there are explicit
finite families $F$ for which the restriction on the adversary in Scenario~1
strictly benefits the learner over Scenario~2, and we quantified how large the
gap $\opt^*_{p,\mathbb{R}}(F)/\opt'_{p,\mathbb{R}}(F)$ can be as a function of
$|F|$. We further established that, for the classes $\mathcal{G}_q$, changing
the distance threshold in Scenarios~1 and~2 from $1$ to an arbitrary radius
$R>0$ affects the minimax $p$-loss only through an explicit multiplicative
scaling factor.
Specifically, we proved an exact scaling law: replacing
the unit distance threshold by an arbitrary radius $R>0$ multiplies the minimax
$p$-loss by $R^{(q-1)p/q}$.
Thus all qualitative phase transitions identified for radius $1$ persist
unchanged at every scale.

For Scenario~3 we proved sharp results illustrating both the power and the
subtleties of distance-weighting.
In particular, for $(p,q)=(2,2)$ we obtained the exact value
$\opt^{\id}_{2,\mathbb{R}}(\mathcal{G}_2)=1$, showing that an identity-type
penalty (equivalently, a $1/\delta$ discount of squared error) exactly matches
the ``action'' budget and yields a finite, tight worst-case bound.
More generally, for positive weight functions $f$ we characterized the optimum
in terms of $\sup_{x>0} x f(x)$, and recovered the sharp value
$\opt^{\exp,c}_{2,\mathbb{R}}(\mathcal{G}_2)=1/(ce)$ for exponential discounting.
These results provide a concrete template: in the unbounded setting, an
appropriate distance-weighted loss can restore a bounded, information-theoretic
mistake guarantee that is quantitatively controlled by the tail behavior of the
weighting function.

While the one-dimensional class $\mathcal{G}_q$ admits finite
opt-values under Scenarios~1--3, our results for the slice class $\mathcal{G}_{q,d}$
show that this behavior does not extend to higher dimensions on an unbounded
domain.
For every $d\ge 2$, the slice constraint allows the adversary to repeatedly
exploit fresh regions of the domain while remaining locally admissible, forcing
infinite loss even under strong locality restrictions.

The results here suggest several natural directions for further work. While $\linint$ and $\linint'$ are natural and effective, it remains open to
classify optimal strategies in the various scenarios.
Are there families $F$ or parameter regimes $(p,q)$ for which interpolation is
suboptimal?
More generally, can one characterize minimax-optimal learners and minimax-optimal
adversaries for
Scenario~3 weightings?

A natural alternative to the slice class $\mathcal{G}_{q,d}$ is to impose a global constraint on the full gradient, for example
\[
\widetilde{\mathcal{G}}_{q,d}
\;=\;
\Bigl\{ f:\mathbb{R}^d \to \mathbb{R} \ \text{absolutely continuous} \;:\;
\int_{\mathbb{R}^d} \|\nabla f(x)\|_q^q \, dx \le 1 \Bigr\}.
\]
This class rules out the coordinate-separable ``local bump'' constructions used in the proof of Theorem~\ref{thm:Gqd-inf-all}, and thus the multivariable impossibility result for $\mathcal{G}_{q,d}$ does not extend to $\widetilde{\mathcal{G}}_{q,d}$ by the same argument.
It is therefore a natural candidate for a smoothness model on $\mathbb{R}^d$ under which finite minimax loss might be achievable in dimensions $d \ge 2$ under Scenarios~1--3 or suitable weighted losses.
Determining whether such bounds in fact hold for $\widetilde{\mathcal{G}}_{q,d}$, and how they depend on $d$, $p$, $q$, and the choice of locality model, remains an open problem.

Scenarios~1 and~2 impose a fixed radius threshold (distance $1$).
A natural refinement is to allow a time-varying or data-dependent radius
$\rho_t$ (e.g.\ decreasing with $t$, or chosen adaptively by the learner) and to
analyze the tradeoff between exploration rate and guaranteed loss. A complementary direction is to impose \emph{computational} constraints,
by capping the number of arithmetic operations permitted per round, and to ask
how smoothness assumptions interact with such caps on an unbounded domain.
Recent work initiated the systematic study of mistake-bounded online learning
with operation caps, establishing general bounds on the minimum per-round
arithmetic required to learn a function family with finitely many mistakes
\cite{GenesonLiTang25}.
It would be interesting to develop an analogous theory for smooth real-valued
function classes on $\mathbb{R}^n$.

\section*{Acknowledgements}
JG was supported by the Woodward Fund for Applied Mathematics at San Jose State University, a gift from the estate of Mrs. Marie Woodward in memory of her son, Henry Tynham Woodward. He was an alumnus of the Mathematics Department at San Jose State University and worked with research groups at NASA Ames. GPT-5 was used for proof development, exposition, and revision.


\end{document}